\documentclass[letterpaper, 10 pt, conference]{ieeeconf} 

\title{Online Flocking Control of UAVs with Mean-Field Approximation}
\author{Paper ID: }
\date{December 2019}
\IEEEoverridecommandlockouts                              % This command is only needed if 
\usepackage[utf8]{inputenc}
\usepackage{amsmath}
\usepackage[Symbol]{upgreek}
\usepackage{graphics}
\usepackage[tight]{subfigure}
\usepackage{hyperref}
\usepackage[ruled,vlined]{algorithm2e}
\usepackage{comment}
\usepackage{todonotes}

\usepackage{amssymb}
\usepackage{amsfonts}

\usepackage{amsthm}
\usepackage{blindtext}
\usepackage{cite}

\usepackage{scalerel}
\usepackage{tikz}
\usetikzlibrary{svg.path}

\definecolor{orcidlogocol}{HTML}{A6CE39}
\tikzset{
  orcidlogo/.pic={
    \fill[orcidlogocol] svg{M256,128c0,70.7-57.3,128-128,128C57.3,256,0,198.7,0,128C0,57.3,57.3,0,128,0C198.7,0,256,57.3,256,128z};
    \fill[white] svg{M86.3,186.2H70.9V79.1h15.4v48.4V186.2z}
                 svg{M108.9,79.1h41.6c39.6,0,57,28.3,57,53.6c0,27.5-21.5,53.6-56.8,53.6h-41.8V79.1z M124.3,172.4h24.5c34.9,0,42.9-26.5,42.9-39.7c0-21.5-13.7-39.7-43.7-39.7h-23.7V172.4z}
                 svg{M88.7,56.8c0,5.5-4.5,10.1-10.1,10.1c-5.6,0-10.1-4.6-10.1-10.1c0-5.6,4.5-10.1,10.1-10.1C84.2,46.7,88.7,51.3,88.7,56.8z};
  }
}

\newcommand\orcidicon[1]{\href{https://orcid.org/#1}{\mbox{\scalerel*{
\begin{tikzpicture}[yscale=-1,transform shape]
\pic{orcidlogo};
\end{tikzpicture}
}{|}}}}
\usepackage{hyperref}

\usepackage{eso-pic}
\newcommand\AtPageUpperMyright[1]{\AtPageUpperLeft{
 \put(\LenToUnit{0.5\paperwidth},\LenToUnit{-1cm}){
     \parbox{0.5\textwidth}{\raggedleft\fontsize{10}{11}\selectfont #1} }
 }}
\newcommand{\conf}[1]{
\AddToShipoutPictureBG*{
\AtPageUpperMyright{#1}
}
}

 \providecommand{\norm}[1]{\lVert#1\rVert}
\graphicspath{{./figs/}}

\DeclareMathOperator*{\argmax}{arg\,max}

\title{\LARGE \bf
Online Flocking Control of UAVs with Mean-Field Approximation
}

\author{Malintha Fernando
\thanks{Malintha Fernando is with Luddy School of Informatics, Computing, and Engineering at Indiana University, Bloomington, IN 47405, USA. E-mail:
        {\tt\small ccfernan@iu.edu}.}
\thanks{Video demonstration at: \url{https://youtu.be/KVkNUKgViSg}.}
\thanks{Code implementation: \url{https://github.com/malintha/mean_field_flocking}.}
}
\theoremstyle{definition}
\newtheorem{definition}{Definition}

\theoremstyle{proposition}
\newtheorem{theorem}{Theorem}[]

\begin{document}
\conf{\textcolor{black}{\textsuperscript{\textcopyright} IEEE Intl. Conf. on Robotics and Automation 2021}}
\maketitle

\begin{abstract}
This work presents a novel, inference-based approach to the distributed and cooperative flocking control of aerial robot swarms.
The proposed method stems from the Unmanned Aerial Vehicle (UAV) dynamics by limiting the latent set to the robots' feasible action space, thus preventing any unattainable control inputs from being produced and leading to smooth flocking behavior.
By modeling the inter-agent relationships using a pairwise energy function, we show that interacting robot swarms constitute a Markov Random Field.
Our algorithm builds on the Mean-Field Approximation and incorporates the collective behavioral rules: cohesion, separation, and velocity alignment.
We follow a distributed control scheme and show that our method can control a swarm of UAVs to a formation and velocity consensus with real-time collision avoidance.
We validate the proposed method with physical UAVs and high-fidelity simulation experiments.

\end{abstract}

\section{Introduction}

With recent developments in aerial robotics, large-scale Unmanned Aerial Vehicle (UAV) swarms show potential in numerous application domains, such as environmental monitoring, performing search and rescue operations, and facilitating the burgeoning demand for mobile cellular networks \cite{sharma2016uav, tolstaya2020learning}. 
Extensive literature addresses the problem of controlling large gatherings of UAVs, with varying degrees of scalability, optimality, flight performance, and control architecture. 
However, as pointed out in recent work \cite{honig2018trajectory, luis2020online, zhu2019distributed}, coordinating such formations online while avoiding collisions is still challenging. Further, small-scale UAVs have limited computational capabilities, making it extremely difficult to perform computationally expensive controlling and planning on board.
% Even though approaches such as \cite{zhu2019distributed} show the capability to split and merge teams of robots online, macro-level controlling remains in the primary stages. 

These challenges have inspired researchers to investigate the functionality of biological swarms that richly demonstrate collective behavior while adhering to physical and dynamic constraints \cite{olfati2006flocking}.
% Such specific natural phenomena arise in animal collectives that have a common group objective, are known as flocking behaviors \cite{olfati2006flocking}.
% Further, other highly sought after robotic swarm characteristics, such as scalability in terms of the number of agents is richly found among biological flocks. 
% Thus, incorporation of self-assembly characteristics can massively benefit aerial swarms in the future \cite{olfati2006flocking}.
% Therefore, adaptation of collective behaviors of biological swarms can massively benefit the aforementioned engineering applications .
It has long been known that the flocking behavior of birds can be accounted for by a set of simple rules, namely  \textit{cohesion}, \textit{separation} and \textit{velocity alignment} \cite{reynolds1987flocks} \cite{okubo1986dynamical}.
% In \cite{vasarhelyi2014outdoor} and \cite{vasarhelyi2018optimized}, the authors demonstrated the flocking behaviour with physical aerial robots for outdoor environmental settings, though the approaches are heavily dependent on tuning and learning.
% However, the approaches were highly susceptible to tuning and training, and reported unavoidable oscillatory motions in hardware during the trajectories \cite{vasarhelyi2014outdoor}.
% Furthermore, as the swarms grow in size and change over time dynamically, the agents generate ad-hoc \textit{local neighborhoods}; thus, centralized controlling becomes tedious. 
% In this work, we seek to develop a distributed and real-time flocking controller for aerial robots that stems from their differentially flat dynamics and yields formation and velocity consensus. 
% More specifically, we incorporate the behavioral swarming rules introduced in former works with quadrotor dynamics to develop a novel flocking controller.
We propose a novel framework based on statistical inference to incorporate swarming rules to coordinate multiple robots.
Specifically, our approach stems from the \textit{differentially flat} dynamics of quadrotors and yields a formation and velocity consensus. 
% Specifically, we address distributed controlling of UAV swarms in real-time while avoiding the collisions.
Such a scheme bolsters the behavioral rules to the robots' dynamics and thus significantly improves the feasibility of the trajectories by averting any unattainable control inputs.
Since the outcome space follows the robots' dynamic model, it greatly reduces the learning and tuning efforts needed.
% Furthermore, in a flocking algorithmic perspective, our method does not depend on neither a leader follower approach \cite{mercado2013quadrotors} nor specific type of swarming motion \cite{levine2000self} \cite{shishika2014male} to reach consensus.

Myriad literature from different domains, such as particle phase separation \cite{blume1971ising}, image segmentation \cite{krahenbuhl2012efficient}, and Simultaneous Localization And Mapping (SLAM) \cite{shankar2020mrfmap}, amply explore Markov Random Field (MRF)-based approaches to model agent interactions in different systems.
MRFs provide a compelling framework for describing the behavioral rules underlying flocking.
% Consequently, we observe that the relations within robot swarms posed by behavioral rules highly compels to the very framework. 
% Consequently, it is highly compelling to model the relations posed by behavioral rules within robot swarms using the very framework.
Therefore, the key idea of this work subsumes the notion that a neighborhood of communicating robots in a swarm constitutes an MRF. 
By modeling the interactions among robots using a
Self-Propelled Particle (SPP)-based energy function, we define a domain for the random variables induced by the robots’ control actions.
% Modeling of interacting entities such as ionized particles or image pixels using MRFs with pairwise energies has been widely explored in phase separation \cite{blume1971ising}, image segmentation \cite{krahenbuhl2012efficient}, Simultaneous Localization And Mapping (SLAM) \cite{shankar2020mrfmap}.
% Further, a considerable amount of literature formulates the robots swarm coordination problem with interacting neighborhood models \cite{gazi2013lagrangian, xi2006gibbs, dong2016time}, and thus provides us a compelling interface to combine with MRFs.
% Briefly, we represent any robot in a local neighborhood with a discrete random variable in a Markov network. 
% By using a linear time-invariant state-space model, exploiting differentially flat dynamics of UAVs, we induce the domain of the random variables with a finite set of control actions. 
We propose inference on the Markov network to approximate the \textit{best} control input that minimizes the \textit{energy} inside the neighborhood. 
Since our method's coordination and collision avoidance depend on reacting to the neighborhoods online, we emphasize the necessity of a computationally tractable inference method.
Thus, we propose using an approximate yet fast variational inference-based Mean-Field Approximation (MFA) to keep the computational effort bounded.
% To achieve scalability, we run instances of the proposing \textit{flocking algorithm} on each communicating agent in the swarm, eliminating the need for centralized control.
% Due to the distributed scheme of our approach, the proposed method is highly scalable in the size of the swarm.

% We show that our method can simulate the flocking behavior of multiple UAV robots without explicit collision avoidance constraints, tuning of parameters, or training. 
% We quantitatively evaluate the resulting trajectories to existing particle flocking algorithms and show that our method agrees with the basic flocking rules. 
The main contribution of our work is a novel, distributed, flocking algorithm that combines differentially flat dynamics.
% We show that this method results in formation and velocity consensus for a group of robots only by utilizing local information. 
We validate the consensus reaching procedure of our method with simulations and physical robot experiments. 
Further, we provide evaluation results for the convergence of the proposed algorithm for varying neighborhood sizes and a comparison of the velocity consensus reaching process against the Vicsek model \cite{vicsek1995novel}.
In contrast to existing UAV swarm formation control methods \cite{honig2018trajectory}  \cite{zhu2019distributed} \cite{alonso2017multi}, our method does not require explicit collision constraints or extensive training \cite{tolstaya2020learning} \cite{vasarhelyi2018optimized}. Furthermore, due to the fast converging nature, our method is suitable for use in real-time aerial swarm coordination.

\section{Related Work}

% \begin{figure}[t]
%   \centering
%     {\includegraphics[width = 0.47\textwidth, clip ]{figs/cover.png}}
% \caption{Initial (left) and reached consensus formations (right) for five Crazyflie nano-drones.}
% \label{fig:cover}
% \vspace{-0.5cm}
% \end{figure}

% \MF{Needs to be rewritten}

% about ising model and the term of energy. being used to find the minimum energy configurations in atoms and molecules. 

% strike a balance between drone team navigation/swarm behaviour generation and MRF based approaches

Numerous studies have been conducted on navigation and formation control of aerial robot swarms of varying sizes in indoor and outdoor environments \cite{honig2018trajectory, zhu2019distributed, fernando2019formation, turpin2013trajectory, 8276634}. Even though the state-of-the-art aerial swarms contain thousands of members, the trajectories are mainly precomputed and stored before execution \cite{vasarhelyi2018optimized}. 
Convex programming-based formation control methods are proposed in \cite{luis2020online} \cite{zhu2019distributed} and \cite{alonso2017multi} with online trajectory computations. However, none of the methods have demonstrated fundamental flocking behaviors, though \cite{zhu2019distributed} shows the capability to split and merge robot teams during the navigation. In \cite{shishika2017mosquito}, the authors used the collective motion of mosquitoes to derive a real-time pursuit law for quadrotors.

On the other hand, evolutionary algorithms \cite{vasarhelyi2018optimized} and graph neural network-based methods \cite{tolstaya2020learning} have shown promise in achieving emergent behavior in quadrotor swarms with local communications. The authors have demonstrated flocking with cohesion and velocity consensus behaviors, yet the approaches are greatly susceptible to tuning and learning. V{\'a}s{\'a}rhelyi et al. \cite{vasarhelyi2014outdoor} reported unavoidable oscillations in the real hardware during the navigation incurred by the flocking controller.
Further, \cite{hu2020vgai} and \cite{schilling2019learning} propose data-driven approaches for the vision-based flocking of quadrotors. In \cite{xi2006gibbs} and \cite{fernando2020swarming}, the authors proposed an MRF-based approach to coordinate a swarm of robots. However, they discuss neither flocking consensus nor the dynamics of the particles.

Numerous methods have been proposed in the swarming community to simulate self-organizing behavior with propelled particle dynamics \cite{tanner2003stable, saber2003flocking, gazi2013lagrangian}. Many such methods suffer from issues of unbounded collision avoidance forces \cite{gazi2013lagrangian} and irregular fragmentation  \cite{olfati2006flocking} \cite{tanner2003stable}. 
These limitations, especially the former, can lead highly agile aerial vehicles to crash.
Eliminating such unfeasible inputs while preserving the consensus is not trivial and highly dependent on tuning parameters. To address this issue, we tailor the control input space to suit the robots' physical limitations by eliminating the possibility of producing dynamically unfeasible trajectories. 
We place our work in the conjuncture of aerial robot formation control, flocking algorithms, and multi-agent planning algorithms. 

\section{Preliminaries}
% This section presents some background on the MRF we are using in this work, followed by trajectory generation with motion primitives on a differentially flat dynamics model.
\subsection{Markov Random Fields}
% what is a markov network?
% why markov network?

Markov networks or MRFs represent the relationships among interacting random variables when the directionality among them is irrelevant.
In this work, we assume that there exist symmetrical interactions among homogeneous robots in a \textit{local neighborhood}, and these constitute the framework of undirected graphical models. We start by defining the concept of factors in an MRF.

% Further, in this work, we make use of the fact that the joint probability distribution of a MRF can be denoted as a joint of the local interactions. 
Let $X = \{X_1, ..., X_N\}$ be a set of random variables in an undirected graphical model. 
\begin{definition}
Let $D \subset X$. We define factor $\phi$ as a generic function that maps the values of $D$ to a real number. Thus, $\phi : \mathrm{Value}(D) \xrightarrow{} \mathbb{R}$.
\end{definition}
% In other words, a factor maps assignments of random variables in $D$ to a real value. 
We now define the joint probability distribution for $X$ parameterized by a set of factors $\Phi$.

\begin{definition}
By considering a complete factorization $\Phi$ of distribution $X$, over subsets $D_1,..,D_M$, we define the joint probability function as
\begin{equation}
    P_\Phi(X_1,...,X_N) = \frac{1}{Z}\Tilde{P}_\Phi(X_1,...,X_N),
    \label{eq:jointP}
\end{equation}
where $\tilde{P}_\Phi(X_1,...,X_N) = \phi_1(D_1) \times \phi_2(D_2) ... \times \phi_M(D_M)$ and $Z = \sum_{X_1 ... X_N} \tilde{P}_\Phi(X_1,...,X_N)$ is a normalizing constant called the \textit{partition function}.
\end{definition}
% how the potentials are arranged?
% This is also known as \textit{Gibbs distribution} over $X$. 
We now obtain a formulation where we can specify \textit{energies} as factors. By using $\phi(D) = \exp{-\epsilon(D)}$, where $\epsilon$ is an \textit{energy function}, and $\phi(D) > 0$, we rewrite Eq. \eqref{eq:jointP} as
\begin{equation}
    P_\Phi(X) = P(X) = \frac{1}{Z} \exp{\{-\sum_{i}^{M}\epsilon_i(D_i) \}}.
    \label{eq:ExpP}
\end{equation}

Eq. \eqref{eq:ExpP} is also known as the log-linear form of the \textit{Gibbs distribution}, $P_\Phi(X_1,...,X_N)$, defined over the factorization $\Phi$. 
The strict positiveness of the factors ensures that the probability distribution is always positive. 
Further, we use this formulation to incorporate context-specific information of the UAV swarm in terms of energy functions. 
% In the field of statistical physics, this model is also known as the \textit{Ising model}. 
Particularly, Eq. \eqref{eq:ExpP} absorbs the concept that the interacting random variables prefer lower energy configurations, thus ideal for minimizing energies induced by other robots and the environment.

\subsection{Dynamic Model}
We derive time parameterized trajectories defined for a differentially flat dynamic model of a quadrotor and a finite set of control actions. Similar to previous works on motion primitive-based planning \cite{liu2017planning}, the state space of a robot consists of the position and its $(n-1)$-th derivatives. Let $x \in \mathbb{R}^{n\times 3}$ be the state space and $u \in \mathbb{R}^3$ be the control inputs of the dynamic system. We use the notation, $x = [p,\dot{p},...,p^{(n-1)}]^T$, and $u=p^{(n)}$ where $p \in \mathbb{R}^3$ is the position of the robot in $SE(3)$. By considering the differential flatness, we can represent the state space model as a \textit{linear time invariant dynamic} system, $\dot{x} = Ax + Bu$ where,
% \begin{equation*}
    
% \end{equation*}
\begin{equation}
    A = \begin{bmatrix}
0 & I_3 & 0 & \hdots & 0\\ 
0 & 0 & I_3 & \hdots & 0\\ 
\vdots & \ddots & \ddots & \ddots & \vdots\\ 
0 & \hdots & \hdots & 0 & I_3 \\ 
0 & \hdots & \hdots & 0 & 0 
\end{bmatrix}, 
B = \begin{bmatrix}
0 \\ 0
\\ \vdots
\\ 0
\\ I_3
\end{bmatrix}.
\label{eq:flat_dynamics}
\end{equation}

We show that given an initial state $x_0$ and a constant control input $u$, the trajectory of a linear time-invariant system can be represented by a time parameterized polynomial, which is differentiable at least $n$ times. 

Assume time parameterized polynomial form for the resulting trajectory of the position $p(t)$, \[p(t) = \sum_{q=0}^\mathbf{q} d_q \frac{t^q}{q!},\] where $\mathbf{q}$ is the degree of the polynomial and $d_q$ is the $q$-th coefficient.
As the control input is the $n$-th derivative of the position trajectory, by differentiating $n$ times,
\[
u \equiv p^{(n)}(t) = \sum_{q=0}^{\mathbf{q}-n} d_{q+n} \frac{t^q}{q!}.
\]
Considering $u$ is a constant control for the time interval, by analyzing the terms on the right hand side, we get the coefficients $d_{(n+1)} = d_{(n+2)} \dots d_{\mathbf{q}} = 0$ and, $d_n = u$. By integrating this control expression $n$ times with the initial condition $x_0$,
\begin{equation}
    p(t) = u \frac{t^n}{n!} + \dots + acc(x_0)\frac{t^2}{2} + vel(x_0)t + pos(x_0),
    \label{eq:polynomial}
\end{equation}

where $acc(x_0)$, $vel(x_0)$ and $pos(x_0)$ denotes the acceleration, velocity and position components of the initial condition. As the degree of the position trajectory $p(t)$ is $n$, it is at least $n$ times differentiable in $t$ to obtain the higher-order states of the system such as velocity, acceleration, jerk, etc. Eq. \eqref{eq:flat_dynamics} and Eq. \eqref{eq:polynomial} can be used to obtain the equivalent trajectory in state space form.

% The equivalent trajectory of the dynamics system can be written as, 
% \begin{equation}
    % x(t) = \exp{(At)} + \int_0^t \exp{(A(t-\sigma))Bd\sigma}.
% \end{equation}

\begin{figure}[t]
  \centering
\subfigure[]
 	{\includegraphics[width = 0.24\textwidth,trim={0cm 0cm 0cm 0cm}, clip ]{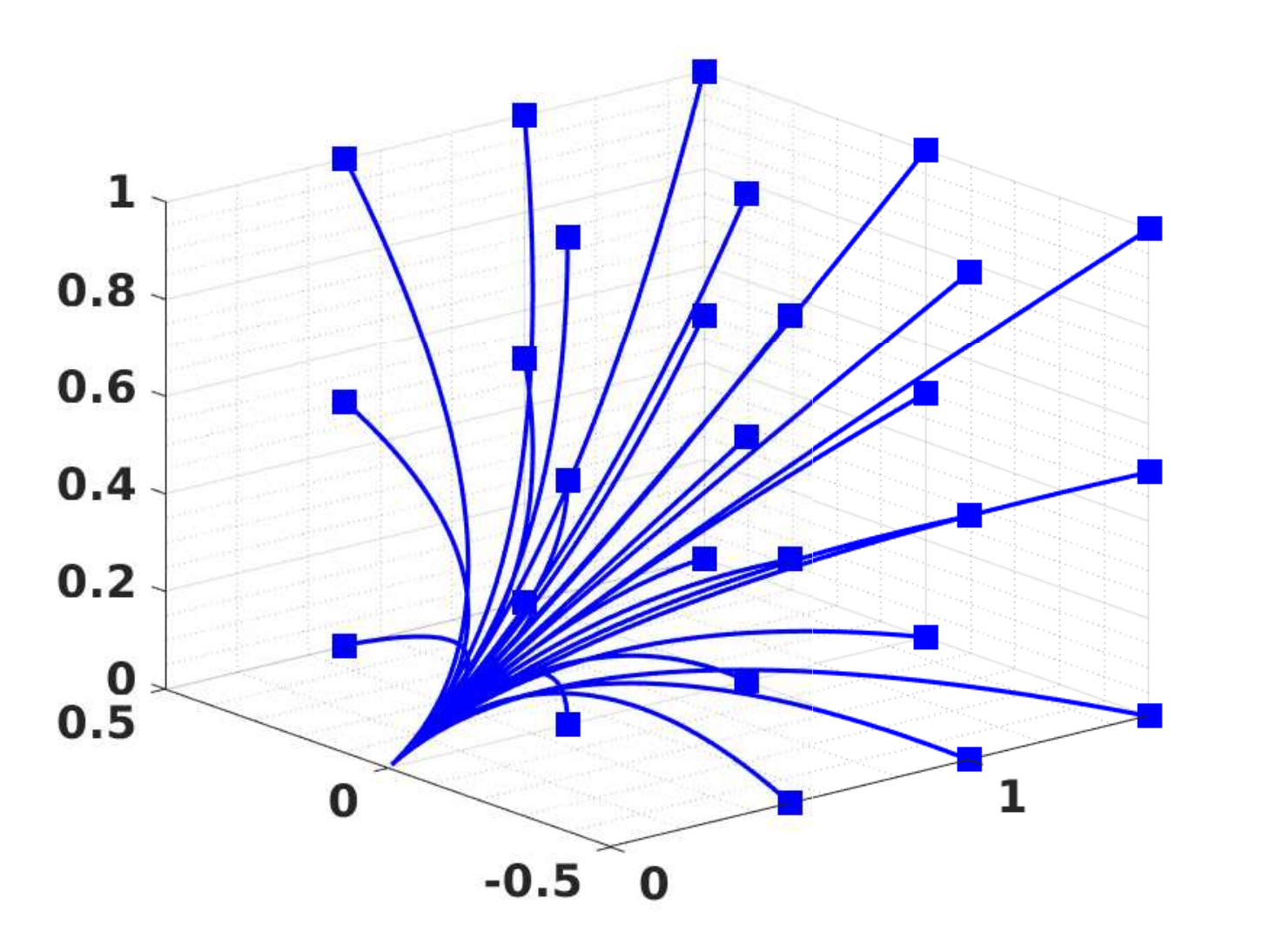}}
%  	\label{as}
%  	\hfill
\subfigure[]
 	{\includegraphics[width = 0.23\textwidth]{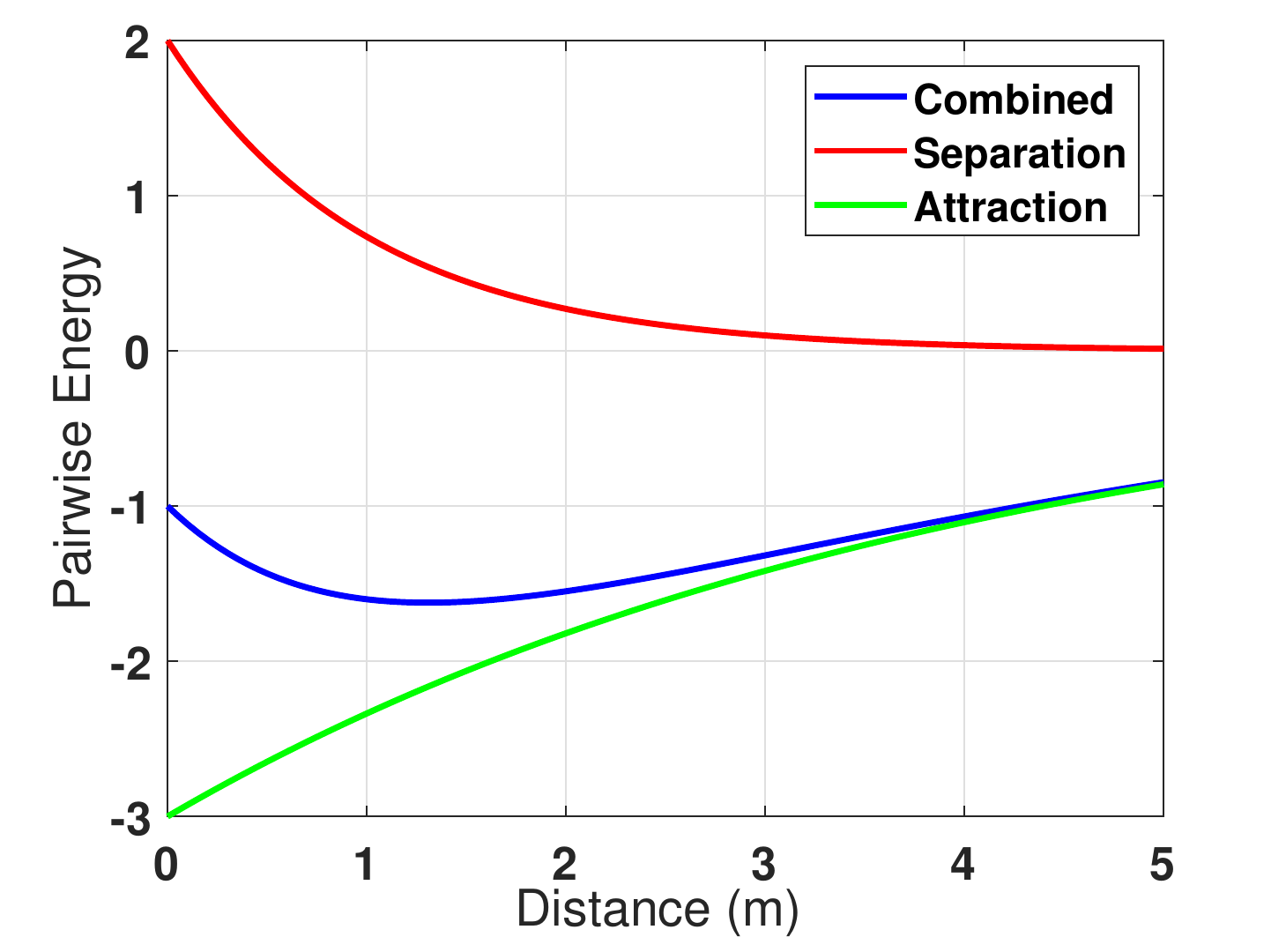}}
%  	\label{fig:trajectories}
\caption{(a). Trajectories constructed for a robot with $vel(x_0)=  [1,0,1]^T$ and $pos(x_0) = [0,0,0]^T$. (b) Morse potential energy function for $k_r =1$, $b = 2$, $a = 3$, $k_a=4$.}
\label{fig:energies}
\vspace{-0.5cm}
\end{figure}

\section{Approach}
In this work, we consider a neighborhood for each robot in the swarm, following the topological swarming model proposed in \cite{ballerini2008interaction}. 
We assume that the robots are capable of periodically obtaining the state of its neighbors, with a planning horizon $T_H$.
% Each robot may construct a MRF, as induced by its neighborhood and perform statistical inference to obtain the control actions for each horizon.
% We begin by introducing the communication topology and resulting local graphs for each robot in the swarm.

\subsection{Robot Swarming Model}
% The literature on collective dynamics has widely studied the concept of the local neighborhood of an individual in a group of animals \cite{mateo2017effect} \cite{hildenbrandt2010self}.
The number of neighbors in one's interaction range directly affects the overall behavior of the swarm \cite{shang2014influence}; however, practical communication limitations hinder expanding the size of the neighborhood to the swarm width. 
% In a study conducted on \textit{murmurations} of \textit{European Starlings} \cite{ballerini2008interaction}, it has been revealed that the collective behavior depends on the topological rather than the metric neighborhoods. 
% Further, in \cite{camperi2012spatially} the authors show that topological  swarming models are more stable and can reach the consensus sufficiently with a minimal number of neighbors. 
The local neighborhood of any robot $i$ in our work consists of its $k$ nearest neighbors. 
We refer to $i$ as the \textit{root} of the neighborhood.
In addition to inter-agent relationships, we substantiate context-specific information to the swarming model in terms of external perturbations (\textit{e.g.}, obstacles, roost, and virtual boundaries).  
% Thus, we 
% We model each the behavior of each individual in the swarm as a result of its social 
% We start with the assume that our robots can communicate with the nearest neighbors during the flight.
% The concept of \textit{local neighborhood} of the flocking behavior of \textit{European Starlings}. 
% Briefly, it has been showed that emergence behavior among birds such as murmuration are outcomes of a few simple rules, namely, \textit{cohesion}, \textit{separation} and \textit{alignment} \cite{hildenbrandt2010self}. 
% The cohesion rule enables the swarming agents to be navigated as a cohesive entity by attracting them toward each other, while the separation rule avoids the inter-agent collisions by introducing a degree of repulsion in between them. 
% The alignment rule is used to match the velocities of nearby agents, so that the neighboring agents will move toward similar directions. 
% In \cite{reynolds1987flocks} showed that it is possible to achieve emergent behavior in computer graphics by applying such rules in the velocity domain of a simple dynamics model. In this work, we also constitute ourselves to the very framework, but with a generalization on the control input. 
% In essence, the collective behavior of animal groups arises as a result of their reactions to internal and external perturbations \cite{okubo1986dynamical}. 
We discuss the quantification of such relationships within a swarm and represent them using energy functions. 
In order to represent different interactions within a swarm, we render two types of energy functions: \textit{pairwise} (for robot-robot interactions) and \textit{unary} (for external perturbations).

% We define a \textit{roosting center}, $C_R$ a centroid of a known volume where the robots are intended to swarm around. 
% In this work, we model aforementioned swarming rules as energy functions that can be incorporated into the Gibbs distribution defined in the previous section. 
% In addition, we limit the swarming of the robots to the roosting volume by introducing another swarming rule. 
% 

\subsubsection{Interaction Energy}

The interaction energy combines the cohesion and separation rules to define a pairwise energy function $\psi_p$ that has a minimum energy at a characteristic length-scale (Fig. \ref{fig:energies}[b]).
Our formulation is similar to the \textit{Morse potential}, which is a widely employed energy function in swarm and statistical physics domains \cite{gazi2013lagrangian} \cite{carrillo2013new}, to model the interactions among SPPs.
For any two neighboring robots, $i$ and $j$, with states $x_i$ and $x_j$, we define $\psi_p$ as
\begin{equation}
    \psi_p(x_i, x_j) = -a\exp{(-{\frac{\mathrm{d}_{ij}}{k_a}})} + b\exp{(-{\frac{\mathrm{d}_{ij}}{k_r}})},
    \label{eq:pairwise}
\end{equation}
where $\mathrm{d}_{ij} = \norm{pos(x_i) - pos(x_j)}$, $k_a,k_r > 0$, $k_a > k_r$, $\frac{b k_r}{ak_a} < 1$. We define the two terms in Eq. \eqref{eq:pairwise} as the \textit{attraction} and the \textit{separation} kernels, respectively. 
The defining characteristics of the Morse potential function depend on constraining the parameters $a$,~$b$,~$k_a$, and $k_r$ as mentioned.
Fig. \ref{fig:energies}(b) shows the change of interaction energy with respect to the distance between two robots. 
Therefore, in the absence of external perturbations, we expect the robots to converge into a consensus formation with the average distance between the robots approximately equal to the local-minima of $\psi_p$.
% Also, we note that the resulting energy is symmetric for both of the agents when the pair is considered independent of the other agents, or in other words a pairwise factor function in a MRF.

\subsubsection{Roosting Energy}
We incorporate perturbations that solely depend on the state of an individual agent and sources that are external to the swarm as a unary energy function, $\psi_u$.
The main purpose of the roosting energy function is to constrain the robots' flocking to a geographical area of interest (\textit{roost}).
In real-world applications, one may design unary functions to represent an area with obstacles: a high-security zone in an urban environment, possibly demarcated by an equipotential line.
% minima, to represent the different points of interests, \textit{i.e.:} analogous to an urban environment with multiple surveillance hot-spots with varying priorities.
% It is also possible to demarcate the geographical boundaries of the \textit{roost} with equipotential lines and use hard constraints to filter out the trajectories that may protrude the robots. 
% In this work, we use an energy function with a global minimum at a given point for simplicity.
% We incorporate perturbations that solely depend on the state of an individual agent and sources that are external to the swarm as a unary potential function, $\psi_u$. 
Assuming that the only external perturbation in an obstacle-free environment is caused by their preference toward the roost, we define the roosting energy of a robot $i$ as a function of the position $pos(x_i)$ and a \textit{roosting center } $C_R \in \mathbb{R}^3$. Therefore,
\begin{equation}
    \psi_{u}(x_i) = \exp\{-{\frac{\norm{pos(x_i) - C_R}}{k_R}}\},
    \label{unary}
\end{equation} 
where $k_R>0$ and $C_R$ is a centroid of a known region in the environment, to which we refer as the roost.
% We can further identify \eqref{unary} as a unary factor energy function.

\subsection{Search Space}
Planning motion in robots using a set of deterministic control actions is a well established concept and used in the literature \cite{liu2017planning} \cite{pivtoraiko2009differentially}.
Typically, the discretized action space is combined with the robot's state space using a state-action graph called \textit{state lattice}, which is later coupled with a graph search to generate motion plans for robots.
However, in multi-robot settings, the dimensionality of such state lattice representations tends to explode with the number of robots, possible actions, and states, making such methods computationally intractable in real time. 
Thus, we limit the search space in our work to a set of finite trajectories defined for a single planning horizon. By limiting the horizon length $T_H \to \delta t$, we achieve reactive and real-time behavior in robots.
The cost of each trajectory is evaluated with respect to the energy functions ($\psi_p$ and $\psi_u$) and the state of each trajectory at the end of the planning horizon, $\delta t$ time.

First, we define a set of deterministic control actions based on the physical limitations of the robots. Let $u_{max}^\textbf{d} \in \mathbb{R}$ be the maximum possible control input for a robot over the axis $\textbf{d}$. Let $\textbf{d}_u$ be a discretization step. Hence, we define $\mu_d = (2u_{max} / \textbf{d}_u + 1)$ number of control actions along dimension $\textbf{d}$, within $[-u_{max}^\textbf{d}, u_{max}^\textbf{d}]$. 
By combining the control actions along the 3 dimensions, we obtain the \textit{control action space} for any robot as $\mathcal{U}$ where the cardinality of $\mathcal{U}$ is $\mu_X\mu_Y\mu_Z$. Note that $\mathcal{U}$ is constant throughout the navigation. 
By exploiting the trajectory generation method presented in the preliminaries, we construct time-parameterized trajectories $l(t)$, for each control input $u \in \mathcal{U}$ and the robot state $x$. Without a loss of generalization, we extend the formulation in Eq. \eqref{eq:polynomial} to be a three-dimensional trajectory for any $u \in \mathbb{R}^3$ and $pos(x) \in \mathbb{R}^3$. For each $l(t) \in \mathbb{R}^3$ and $n=2$, where control inputs lie in the acceleration domain,

% \begin{equation}
%     l(t) = u \frac{t^n}{n!} + \dots + acc(x)\frac{t^2}{2} + vel(x)t + pos(x).
%     \label{eq:l_generic}
% \end{equation}

% By using $n=2$, where the control input lies in the acceleration domain, we rewrite Eq. \eqref{eq:l_generic} as
\begin{equation}
    l(t) = u \frac{t^2}{2} + acc(x)\frac{t^2}{2} + vel(x)t + pos(x).
    \label{eq:l}
\end{equation}

Therefore, for any robot $i$, we define a search space $\mathcal{L}_i$, such that $\mathcal{L}_i(\mathcal{U},x(t),t) = \{l_{i_1}(t),l_{i_2}(t), \dots \}$. Considering the differentially flat dynamics of aerial drones, we track any trajectory $l(t)$ along each dimension independently.

\subsection{Markov Random Field Construction}
% \begin{figure}
%     \centering
% \subfigure[]
%  	{\includegraphics[width = 0.28\textwidth,trim={1.5cm 1.5cm 1.5cm 1.5cm}, clip]{figs/topology.png}}
% \subfigure[]
%  	{\includegraphics[width = 0.18\textwidth,trim={1cm 3cm 1cm 3cm}, clip]{figs/mrf_root.png}}
%     \caption{(a) An instantiated communication topology and local neighborhoods for a group of 10 spatially distributed robots, using $k=3$. The arrowheads denote the directions from root nodes to neighbor nodes. (b) The fully connected Markov random field corresponds to the local neighborhood of the root node 9.}
%      \label{fig:topology}
%      \vspace{-0.5cm}
% \end{figure}
% We start by representing the communication topology of a group of interacting robots. 
Following the topological interaction rule, where each robot communicates with $k$ nearest neighbors, we construct a set of local neighborhoods. 
% Fig. \ref{fig:topology}(a) shows the communication topology and the local neighborhoods for a group of 10 robots. 
At the beginning of a planning horizon, each robot generates a complete (fully-connected) MRF corresponding to its local neighborhood.
The completeness of the model helps us to embed the local information carried by each neighbor into the inference process. 
% Fig. \ref{fig:topology}(b) shows the MRF corresponding to a node in the communication topology.
Formally, let $X = \{ X_1, X_2, \dots\}$ be the random variables of MRF, where $X_i$ is associated with the $i$-th neighbor. Further, $Domain(X_i) = \mathcal{L}_i(\mathcal{U},x_i(t),t)$. It is important to note that the graphical model and the domain of the random variables are instances of a time-varying spatial distribution as the robots move.

Let $G$ be the underlying graph of the MRF with a set of vertices $V$ and edges $\mathcal{E}$ where the the cardinality of $V$ is the number of robots in the local neighborhood.

\begin{definition}
A clique $c$ of a graph $G$ is any complete subgraph of $G$. 
\end{definition}

Let $C$ be a clique factorization of $G$ that consists of singly, -- $c_u$, and doubly, -- $c_p$, connected cliques. Therefore, $c_u, c_p \in C$. By associating previously defined pairwise and unary potential functions with each element in $C$, we define a Gibbs energy function for the local neighborhoods. Let $\Phi$ be the factorization of $X$ defined over $C$, where $\Phi_p$ and $\Phi_u$ are subsets of $\Phi$.
For any $\phi_u \in \Phi_u$ and $\phi_p \in \Phi_p$, using the log linear form,
\begin{subequations}
\begin{equation}
    \phi_u(X_i) = \exp{-\psi_u(x'_i)},
\end{equation}
\begin{equation}
    \phi_p(X_i, X_j) = \exp{-\psi_p(x'_i, x'_j)},
\end{equation}
\label{eq:ll_form}
\end{subequations}
where $X_i,X_j$ are two random variables in the MRF associated with robots $i,j$ and $x'_i$ is the state of robot $i$ after a horizon length $\delta t$.  Formally, $x'_i = [l_i(t+\delta t),\dot{l}_i(t+\delta t), \ddot{l}_i(t+\delta t), \dots]^T$ for any $l_i \in \mathcal{L}_i$.

% Thus, we model the interaction topology in our work as a fully connected graphical model, that is bounded by the sensory limitations of robots. Due to the non-directional (symmetric) behavior of pairwise and unary energy functions, we adopt a Markov random field to model the interaction topology.

% We start by defining a metric called \textit{observable range} for a given robot, a fixed radius $R$ which we assume to be fully observable. We consider robot interactions inside any observation range defines a \textit{sub-swarm}. For any sub-swarm, defined inside an observation range $R_i$, we construct a fully connected interaction topology to represent the robot-robot interactions. This assumption lends the sub-swarm to be treated as a typical swarm. Further, this formulations subsumes the notion that a swarm with a fully connected interaction topology improves the collective behavior. In Fig. \ref{fig:topology} we show the interaction topology for a typical sub-swarm inside an observable range of robot $i$. We represent the pairwise and unary interactions within a given sub-swarm, with a MRF. 
% Next we present our framework for a generalized swarm that consists of $N$ robots, however, this could easily be specified for a sub-swarm defined inside any observation range. Let $N(t)$ be the number of in a swarm at time $t$. We define a set of random variables $X = \{X_1,\dots,X_{N(t)}\}$ with each $X_i$ corresponding to robot $i$. Each random variable defines the state of the robot in the environment. 

\begin{theorem}
The probability of a local neighborhood forming in an environment can be parameterized by a set of pairwise and unary energy functions.
\end{theorem}

\begin{proof}
% Let $G$ be a graph defined over a fully connected interaction topology of $N = N(t)$ robots. 
Using the definition of Gibbs distribution in product form Eq. \eqref{eq:jointP} for the MRF defined over a local neighborhood, 
\begin{equation}
    P_{\Phi}(X) \propto \prod_{\substack{\phi_u \in \Phi_u}}\phi_u (c_u) \prod_{\substack{ \phi_p \in \Phi_p}}\phi_p(c_p).
\end{equation}
Substituting from the log linear form of factor potentials as in \eqref{eq:ll_form},
\begin{equation}
P_{\Phi}(X) \propto \exp\{-\sum_{c_u} \psi_u(c_u) -\sum_{c_p} \psi_p(c_p) \}.
\end{equation}
However, $c_u $ and $  c_p$ are unary and pairwise cliques of $G$, and they consist of vertices and edges of the original graph. Since $G$ is fully connected, we write the above with associated random variables as
\begin{equation}
    P_{\Phi}(X) \propto \exp\{-\sum_{i \leq N} \psi_u(x_i) - \sum_{\substack {i\leq N\\j<i}} \psi_p(x_i,x_j) \}.
    \label{eq:gibbs_prob}
\end{equation}
The summation of $\psi_u$ and $\psi_p$ inside the exponential function produces the \textit{neighborhood energy}. 
It can be seen that the neighborhood energy changes with the value assigned to each random variable.
Each trajectory assignment to the random variables postulates a different robotic formation at the end of the particular time horizon.
As per the formulation of Eq.\eqref{eq:gibbs_prob}, higher neighborhood energies correspond to lower probabilities for the formation. Thus, we can use a Gibbs distribution parameterized by unary and pairwise factors to define the probability of each possible robot formation. \end{proof}
Furthermore, any trajectory distribution over the MRF is analogous to the root robot's predictions on the neighbors' future actions given the local information. 

\subsection{Mean Field Approximation with Velocity Alignment}
Given a fully connected MRF defined over a local neighborhood, we approximate the posterior distribution to obtain the best trajectory assignment that minimizes the neighborhood energy.
As the search space and the neighborhood grow in size, the computational cost of exact inference methods increases by magnitudes and becomes intractable in real-time.
Therefore, we propose using inexact yet fast MFA in this work. 
Further, we incorporate the velocity alignment property into the update rule as a compatibility function.

Briefly, we seek an approximate distribution $Q(X) = \prod_i Q_i(X_i)$ that minimizes the Kullback-Leibler (KL) divergence with the true posterior distribution $P(X)$. 
One prevailing assumption in MFA is the independence of the random variables in the true posterior. 
However, as each planning horizon results in new posteriors, any disparities between the true and the approximated distributions will not be propagated into the future.
By restricting the approximating posterior $Q(X)$ to a valid probability distribution and a product of the independent marginal distributions, we can obtain the following mean-field update rule.
Due to the lengthy derivation of the update rule, we direct any interested readers to the Inference section of \cite{koller2009probabilistic}. 
A disposition of the derivation can also be found in the supplementary material for \cite{krahenbuhl2012efficient}.

% By solving the optimization problem, we arrive at the following mean-field update rule for a fully connected MRF.
% From the definition of KL divergence,
% \begin{subequations}
%% removing the KL divergence equation
\begin{equation*}
    D(Q||P_{\Phi}) = \mathbb{E}_{\mathcal{U} \sim Q} [\ln\frac{Q(\mathcal{U})}{P_{\Phi}(\mathcal{U})}],
\end{equation*}

% \begin{equation}
%      D(Q||P)  = E_{X \sim Q}[\ln Q(x)] + E_{X \sim Q}[\ln P(x)].
% \end{equation}
% From \eqref{eq:gibbs_prob} and the definition of expectation, we write, 
% \begin{equation}
%     -\sum_x Q(x) \ln{Q(x)} = E_{X\sim Q}[-\ln{Q(X)}],
% \end{equation}
% where, $E_{X \sim Q}$ is the expectation of $X$ under distribution $Q$. Using the product form of Gibbs distribution \eqref{eq:jointP} and, the neighborhood energy, \eqref{eq:gibbs_prob},
% \begin{equation}
%     D = E_{X \sim Q}[\ln Q(X)] + E_{X \sim Q} [-\ln \tilde{P}_\Phi(X)] + E_{X \sim Q}[\ln z].
% \end{equation}
% \end{subequations}

% we obtain the following formal optimization problem, where $F[\tilde{P},Q]$ is the \textit{energy functional} \cite{koller2009probabilistic}.

% Let $E_{x \sim Q}[\ln Q(x)] + E_{x \sim Q} [\epsilon(X)] = F[\tilde{P},Q]$ where, $F$ is the \textit{energy functional}. 
% As $E_{x \sim Q}[\ln z]$ is a constant, we minimize the energy functional $F[\tilde{P},Q]$ to obtain $Q(X)$ that minimizes KL divergence. 
% Therefore, we write the optimization problem formally,

\begin{equation*}
  F(\tilde{P}_{\Phi},Q) = \sum_{\phi \in \Phi}\mathbb{E}_Q[\ln \phi] + \mathbb{H}_Q(\mathcal{U})  
\end{equation*}

\begin{subequations}
  \begin{alignat*}{2}
    \mathrm{Find} \quad & \{Q_i(\mathcal{U}_i)\}, \\
    \mathrm{Maximizing } \quad & F(\tilde{P}_{\Phi},Q) =  \sum_{\phi \in \Phi}\mathbb{E}_Q[\ln \phi] + \mathbb{H}_Q(\mathcal{U}),  \\
    \mathrm{Subject\,to} \quad &Q(\mathcal{U}_i) = \prod_i{Q_i(\mathcal{U}_i)}, \notag \\
    & \sum_i{Q_i(\mathcal{U}_i)} = 1.
    \nonumber
  \end{alignat*}
\end{subequations}
% Here, we restrict the approximating posterior $Q(X)$ to be a valid probability distribution and also a product of the marginal distributions of the neighboring robots. By solving the optimization problem, we obtain the following mean-field update rule for a fully connected MRF.

\begin{equation}
    Q_i(u_i) = 
    \frac{1}{Z_i}\exp \{-\psi_u(\dot{x}_i) -\sum_{u_j \in \mathcal{U}_j}\mu(\dot{x}_i,\dot{x}_j)\sum_{j \neq i}  \tilde{Q}_j(u_j)\},
        \nonumber
\end{equation}
\begin{equation}
\tilde{Q}_j(u_j) = Q_j(u_j)\psi_p(\dot{x}_i,\dot{x}_j).
\end{equation}
\label{eq:update}
In order to incorporate the velocity alignment behavioral rule, we define $\mu(x_i', x_j')$ as a \textit{velocity compatibility} function that measures the distance between velocity components of $i $ and $j$ robots' trajectory assignments. 
Therefore,
\begin{equation*}
    \mu(x_i', x_j') = \|vel(x_i') - vel(x_j'))\|.
\end{equation*}

Firstly, we calculate the expected pairwise energy of a trajectory assignment $l_i$ with respect to all the other robots $\sum_{j \neq i}  \tilde{Q}_j(l_j)$.  Secondly, the velocity compatibility function penalizes the expected energy caused by the dissimilarity of the trajectories' velocities. This increases the probability that the robots will have trajectories with likelier velocity components at the end of the time horizon.
We consider the sum of the robots' unary energy $\psi_u(x_i')$ and the weighted, expected pairwise energy as the total Gibbs energy of a given trajectory $l_i$. Finally, we calculate the probability of each trajectory assignment $Q_i(l_i)$ with Gibbs distribution Eq. \eqref{eq:gibbs_prob}. 
% We embed this update process into an algorithm.

\begin{algorithm}
\SetAlgoLined
% Root robot: $r$ \\
 Initialization\\
\For{$i \leftarrow 1 \dots r \dots N$} {
 $\mathcal{L}_i \leftarrow \mathcal{L}_i(\mathcal{U},x_i,t)$ \\
 $Q_i(l_i) \leftarrow \frac{1}{Z_i} \exp \{-\psi_u(\dot{x}_i)\}$ \\
}
Approximation \\
 \While{$Q_{old}(\mathcal{L}) \neq Q(\mathcal{L})$}{
  $Q_{old}(\mathcal{L}) \leftarrow Q(\mathcal{L})$ \\
  $\tilde{Q}_i(l_i) \leftarrow \sum_{j \neq i} Q_j(l_j) \psi_p(\dot{x}_i, \dot{x}_j)$ \\
  $\hat{Q}_i(l_i) \leftarrow \sum_{l_j \in \mathcal{L}_j} \mu(\dot{x}_i, \dot{x}_j)\tilde{Q}_i(l_i)$ \\
  $Q_i(l_i)  \leftarrow \frac{1}{Z_i} \exp \{-\psi_u(\dot{x}_i) -\hat{Q}_i(l_i) \} $ \\
 }
 Execute $\argmax_{l_r} Q_r(\mathcal{L}_r)$ on robot $r$ \\
 \caption{Flocking algorithm for robot $r$}
 \label{algo}
\end{algorithm}
Algorithm \ref{algo} summarizes the key steps of our method into an iterative algorithm based on MFA. Here, we denote the root robot of the particular local neighborhood as $r$ and the size of the neighborhood as $N$. The algorithm runs on each robot $r$ in the swarm simultaneously. The initialization step computes the search space $\mathcal{L}_i$ for each robot in the neighborhood and their probability distributions based on the unary energies. We iterate the approximation step until the target probability distribution of each robot in the neighborhood converges. Finally, the root robot may execute the trajectory that maximizes the approximated marginal probability distribution of the corresponding RV.
% Any resulting trajectory distribution of the MRF is congruous to root robot's beliefs on the neighbors' actions given the local information.
% For detailed derivations of the optimization problem, the update rule and a theoratical convergence analysis of mean field approximation, please refer to \cite{koller2009probabilistic}.

\section{Experiments and Results}
% \subsection{Experimental Setup}

\begin{figure}[]
    \centering
 	{\includegraphics[width = 0.48\textwidth,trim={1cm 1cm 1cm 0.5cm}, clip]{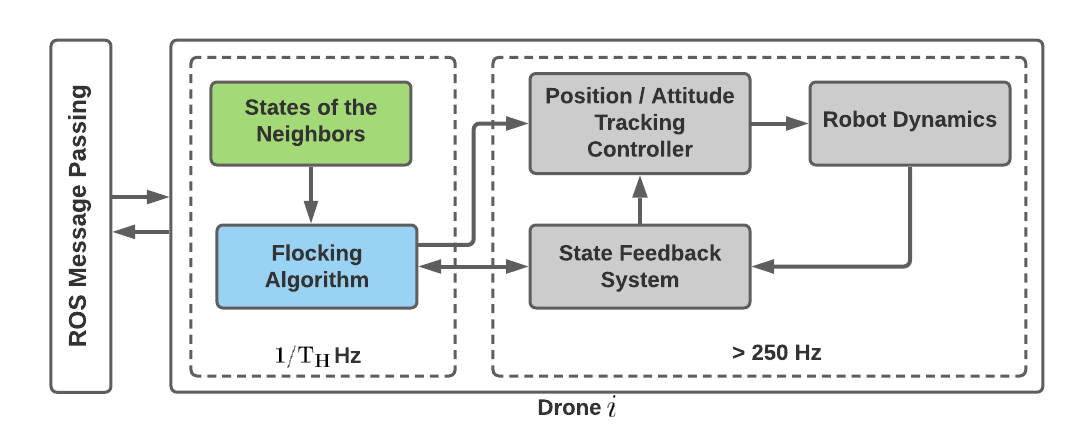}}
    \caption{Overall design of the system. The components related to the robot dynamics and control (grey) runs over 250Hz. The flocking algorithm and the communication component run at a much slower rate.}
     \label{fig:system}
     \vspace{-0.5cm}
\end{figure}

 \begin{figure*}[ht] \vspace{-10pt}
  \centering
  \subfigure[]
 	{\includegraphics[width=0.18\textwidth,trim={5cm 1cm 1cm 2cm}, clip]{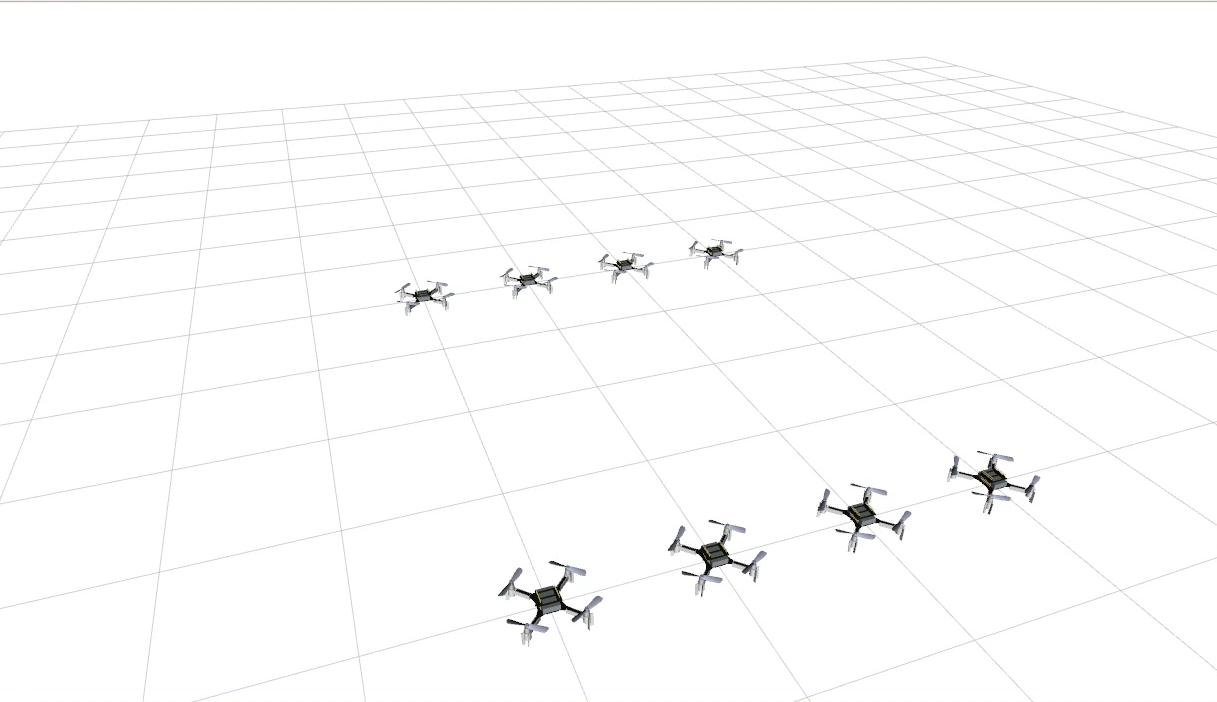}}
    \hfill
\subfigure[]
 	{\includegraphics[width=0.18\textwidth,trim={5cm 1cm 1cm 2cm}, clip]{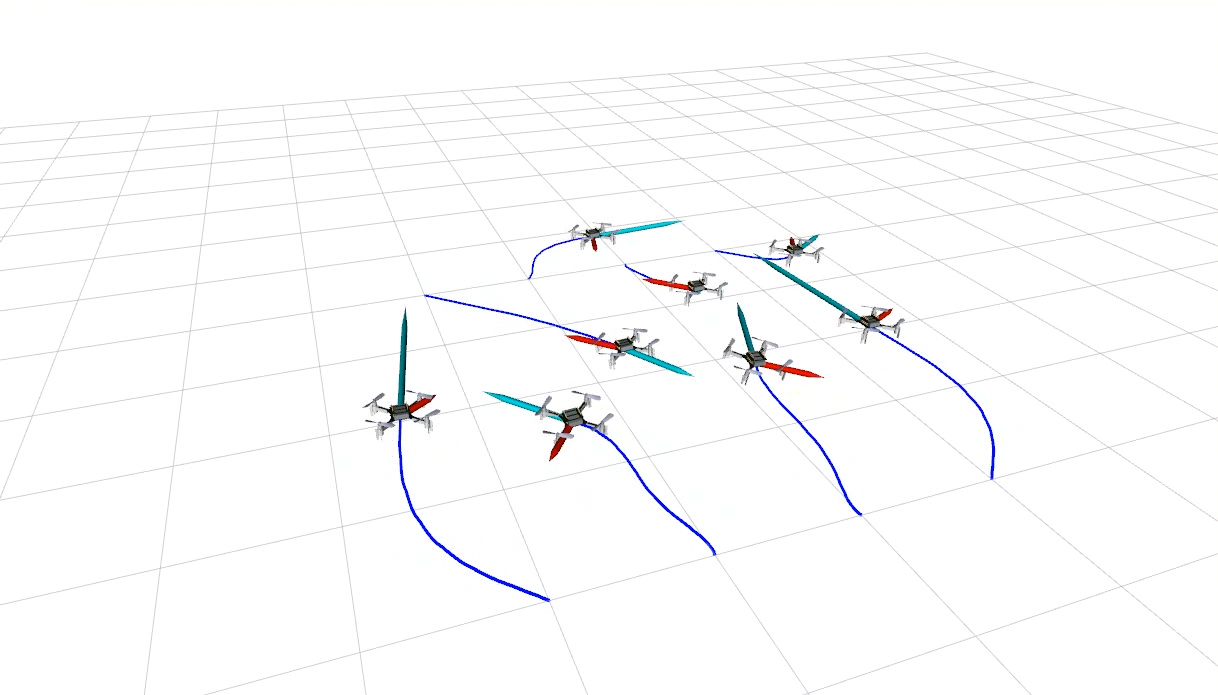}}
 	\hfill
\subfigure[]
 	{\includegraphics[width=0.18\textwidth,trim={4cm 2.5cm 4cm 2cm}, clip]{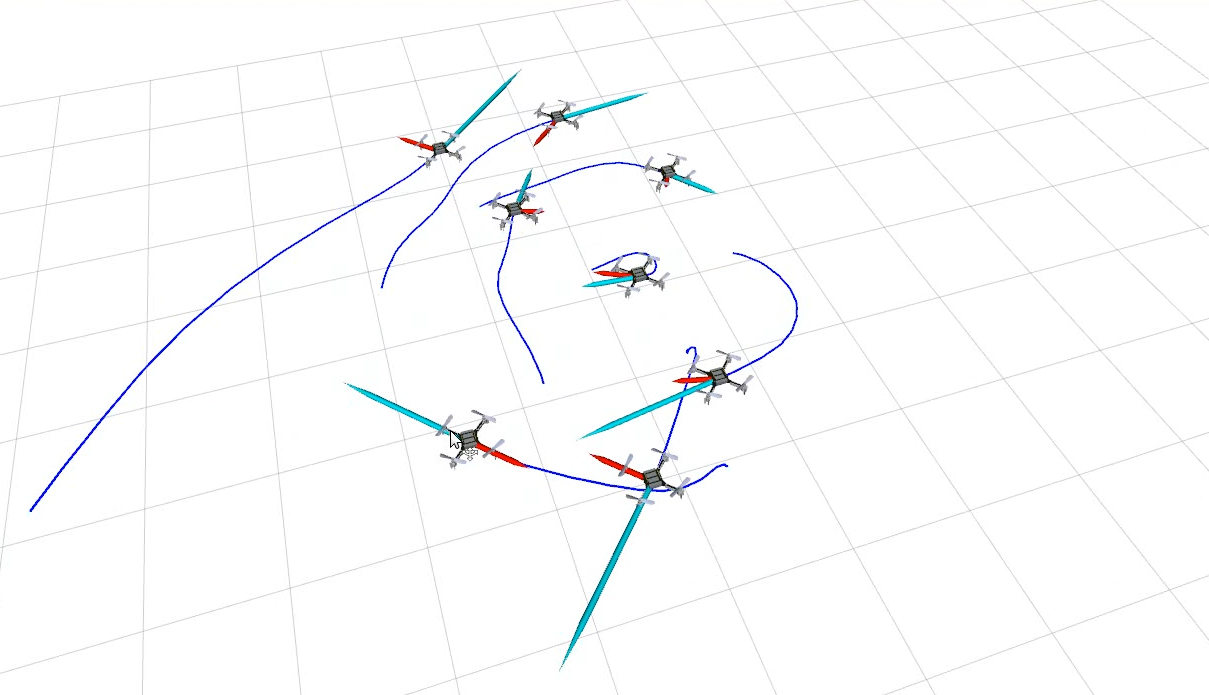}}
 \hfill
 \subfigure[]
 {\includegraphics[width=0.18\textwidth, trim={4cm 1cm 4cm 2cm}, clip]{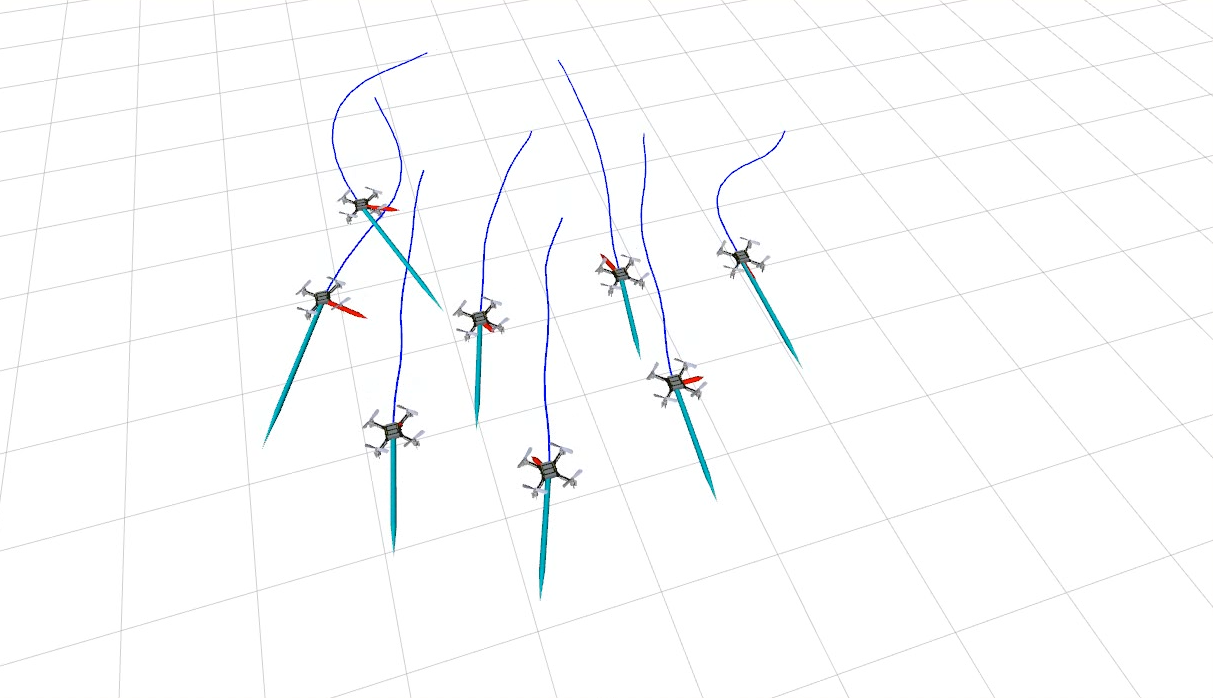}}
\hfill
 \subfigure[]
 {\includegraphics[width=0.18\textwidth, trim={4cm 1cm 5cm 2cm}, clip]{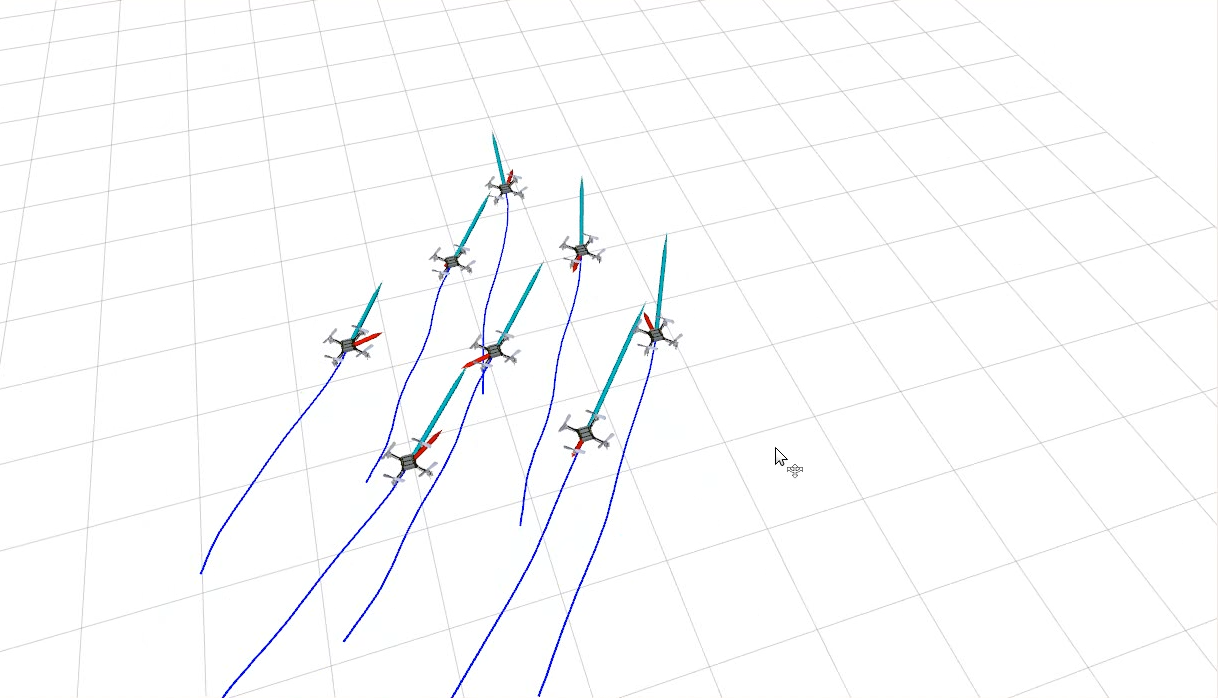}}
 
\caption{Snapshots of a swarm of drones at time (a) 0s, (b) 10s, (c) 20s, (d) 30s, (e) 40s, using $k = 3$. For ease of visualization, the drones are controlled on a 2-dimensional plane of $20m \times 20m$. The blue lines show the trajectory trails of the robots for the last 5s. The red and cyan arrows depict the acceleration inputs and the current velocity of the robots, respectively. The lengths of the arrows are proportional to the size of the vectors in each dimension. }
\label{snapshots}
 \vspace{-0.5cm}
\end{figure*}

\begin{figure}[t]
  \centering
\subfigure[]
 	{\includegraphics[width = 0.23\textwidth, trim={0cm 0.5cm 0cm 0.5cm}]{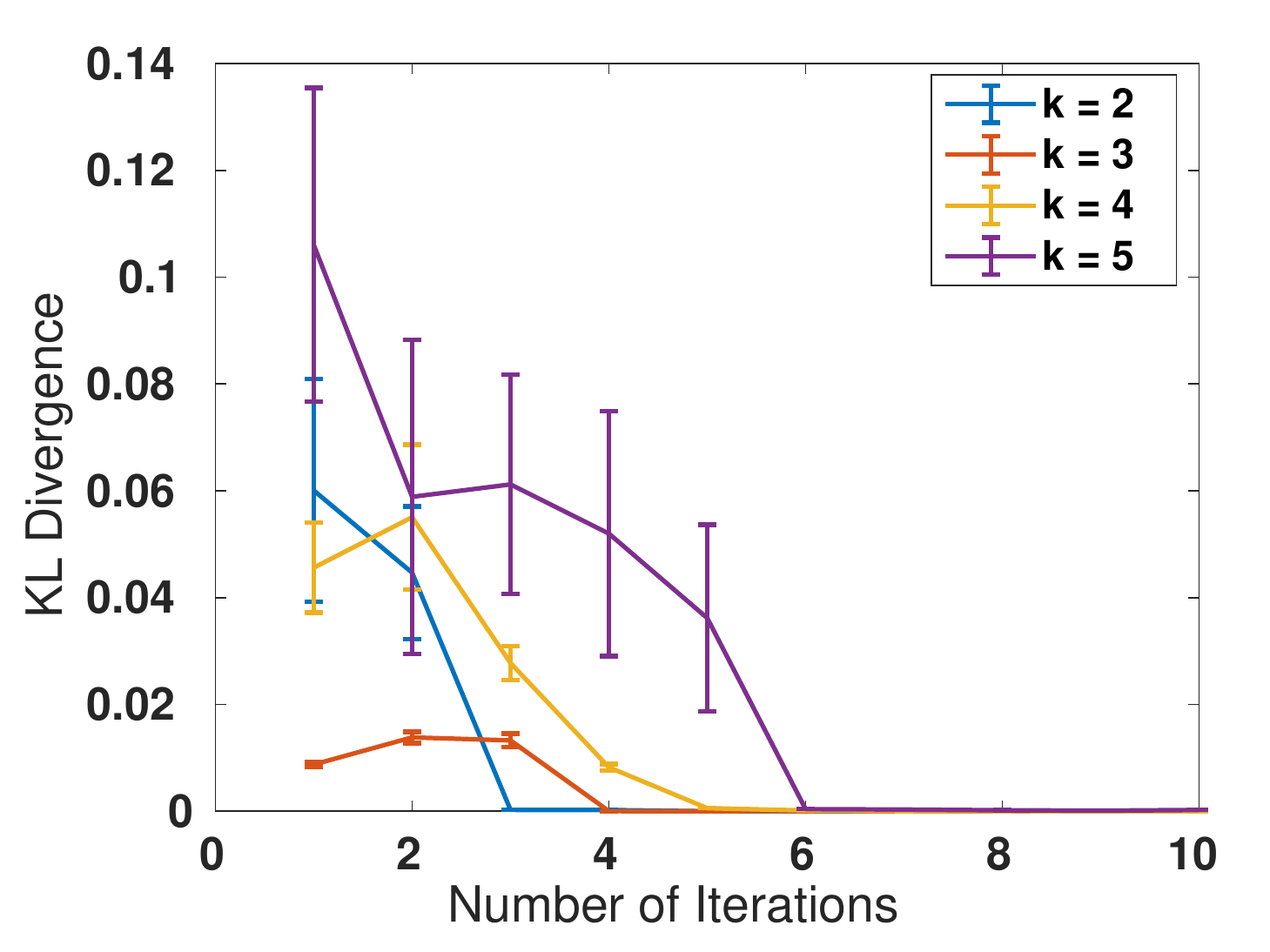}}
\subfigure[]
 	{\includegraphics[width = 0.23\textwidth, trim={0cm 0.5cm 0cm 0.5cm}]{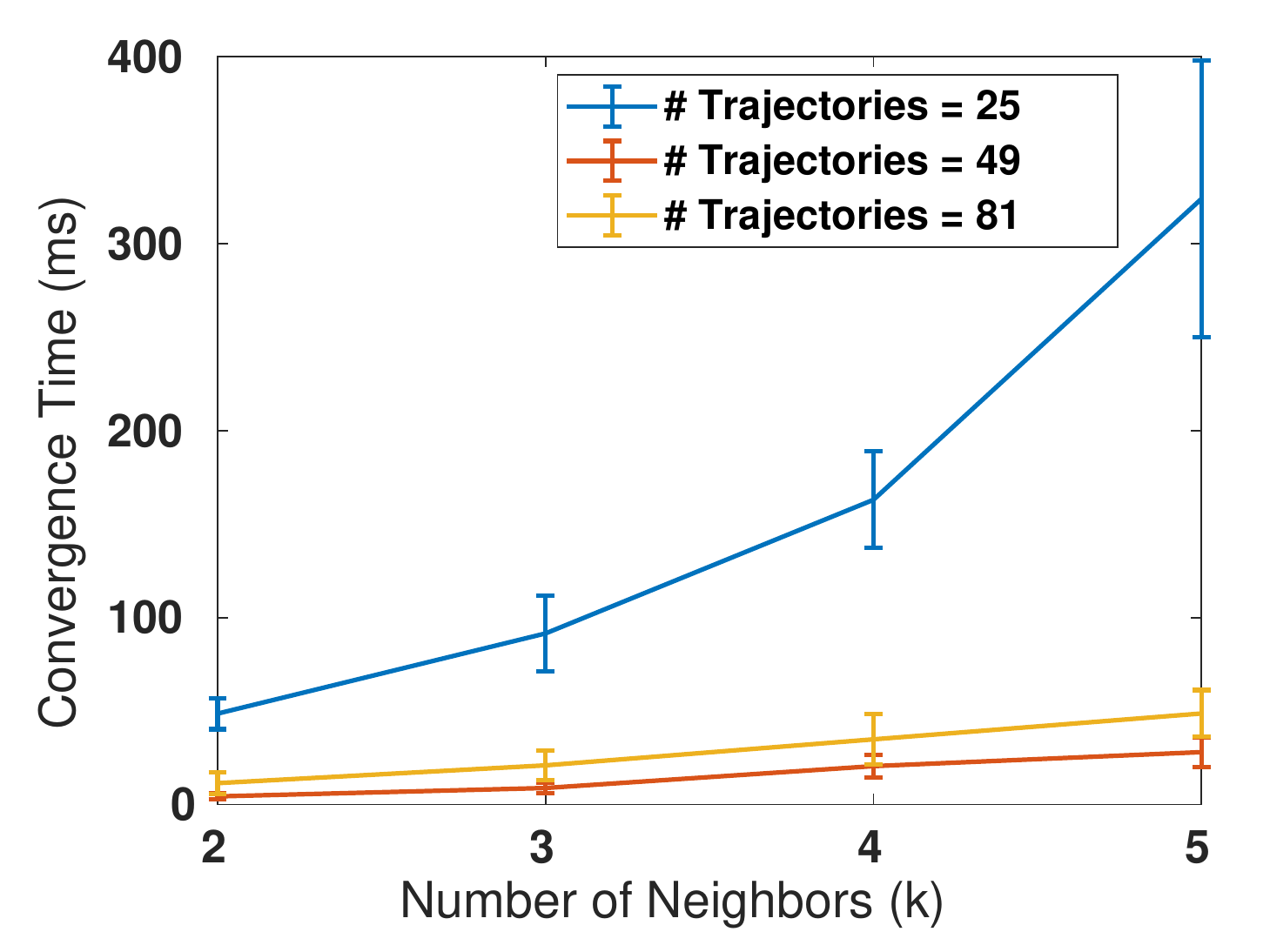}}
\subfigure[]
    {\includegraphics[width = 0.23\textwidth, trim={0cm 0.5cm 0cm 1cm}]{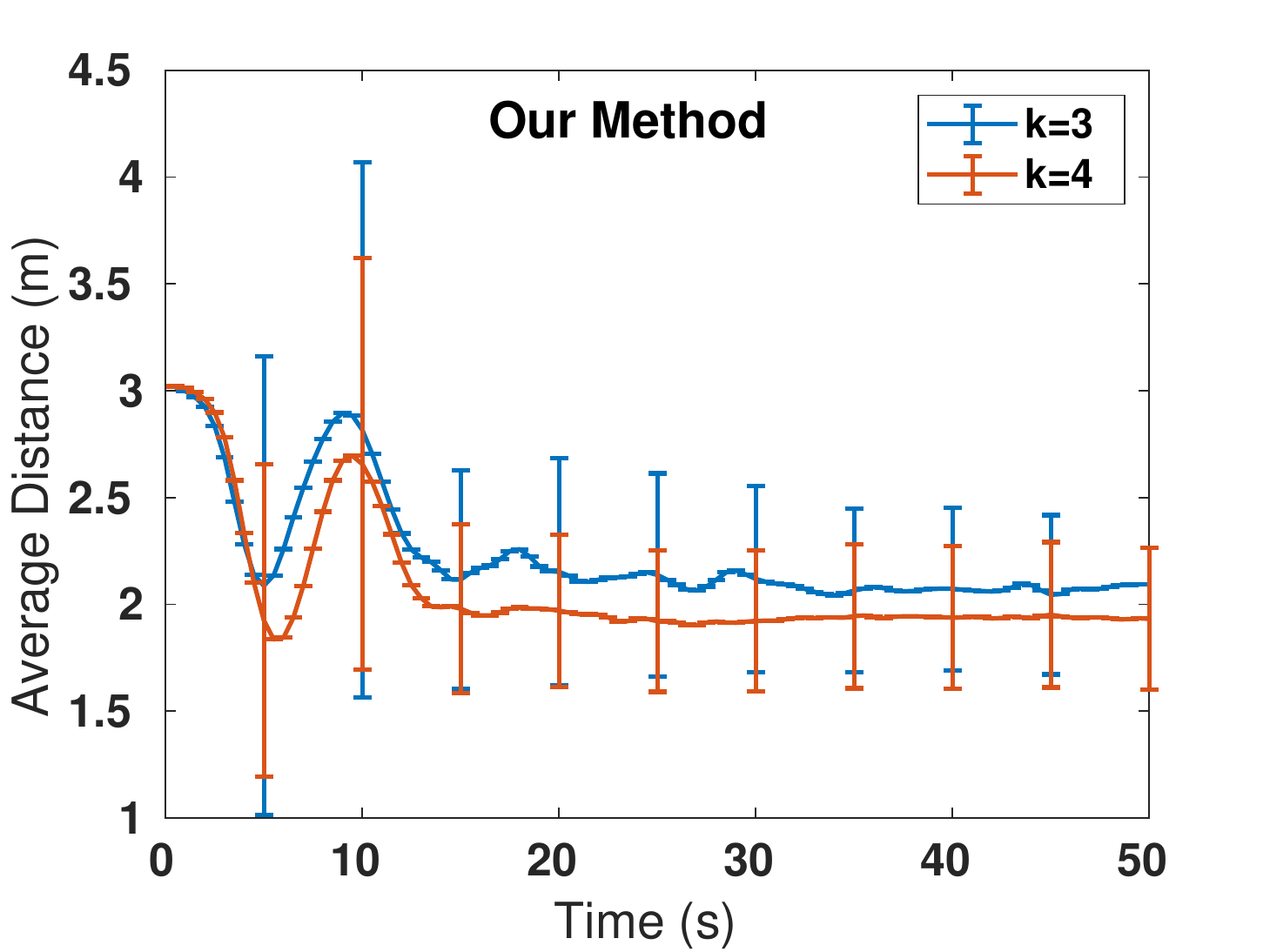}}
\subfigure[]
 	{\includegraphics[width = 0.23\textwidth, trim={0cm 0.5cm 0cm 1cm} ]{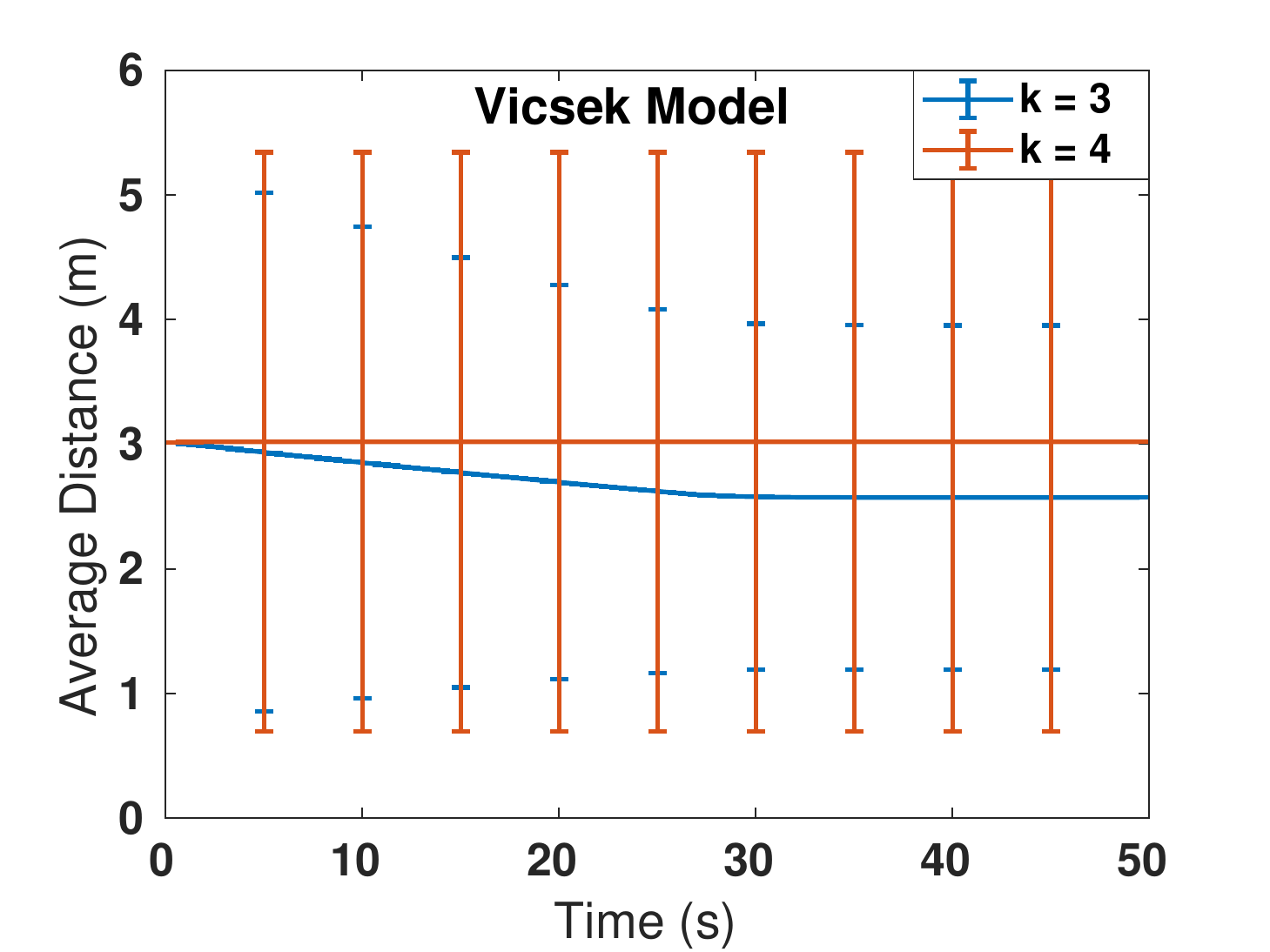}}
 \subfigure[]
    {\includegraphics[width = 0.23\textwidth, trim={0cm 0.5cm 0cm 1cm}]{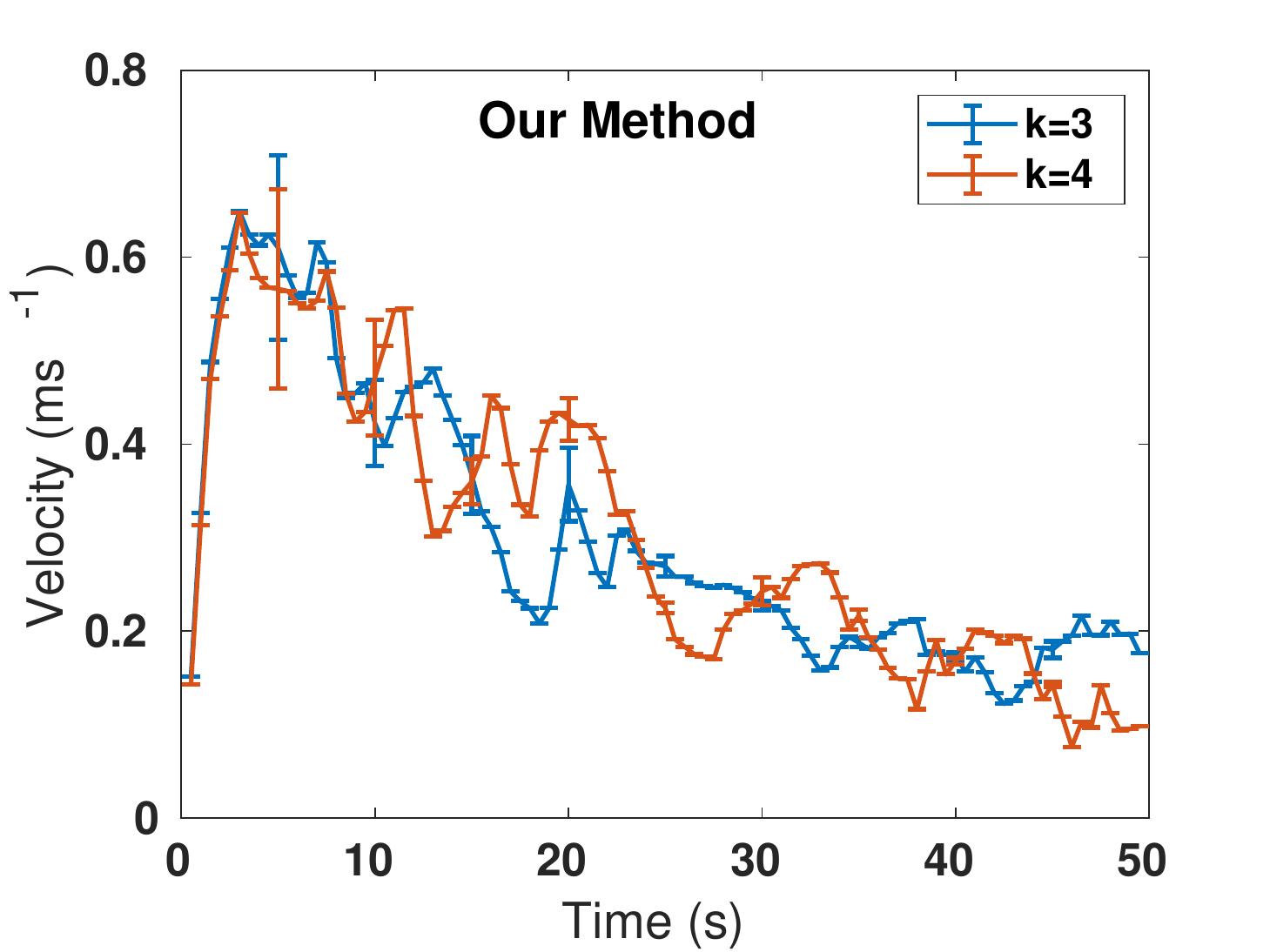}}
\subfigure[]
 	{\includegraphics[width = 0.23\textwidth, trim={0cm 0.5cm 0cm 1cm}]{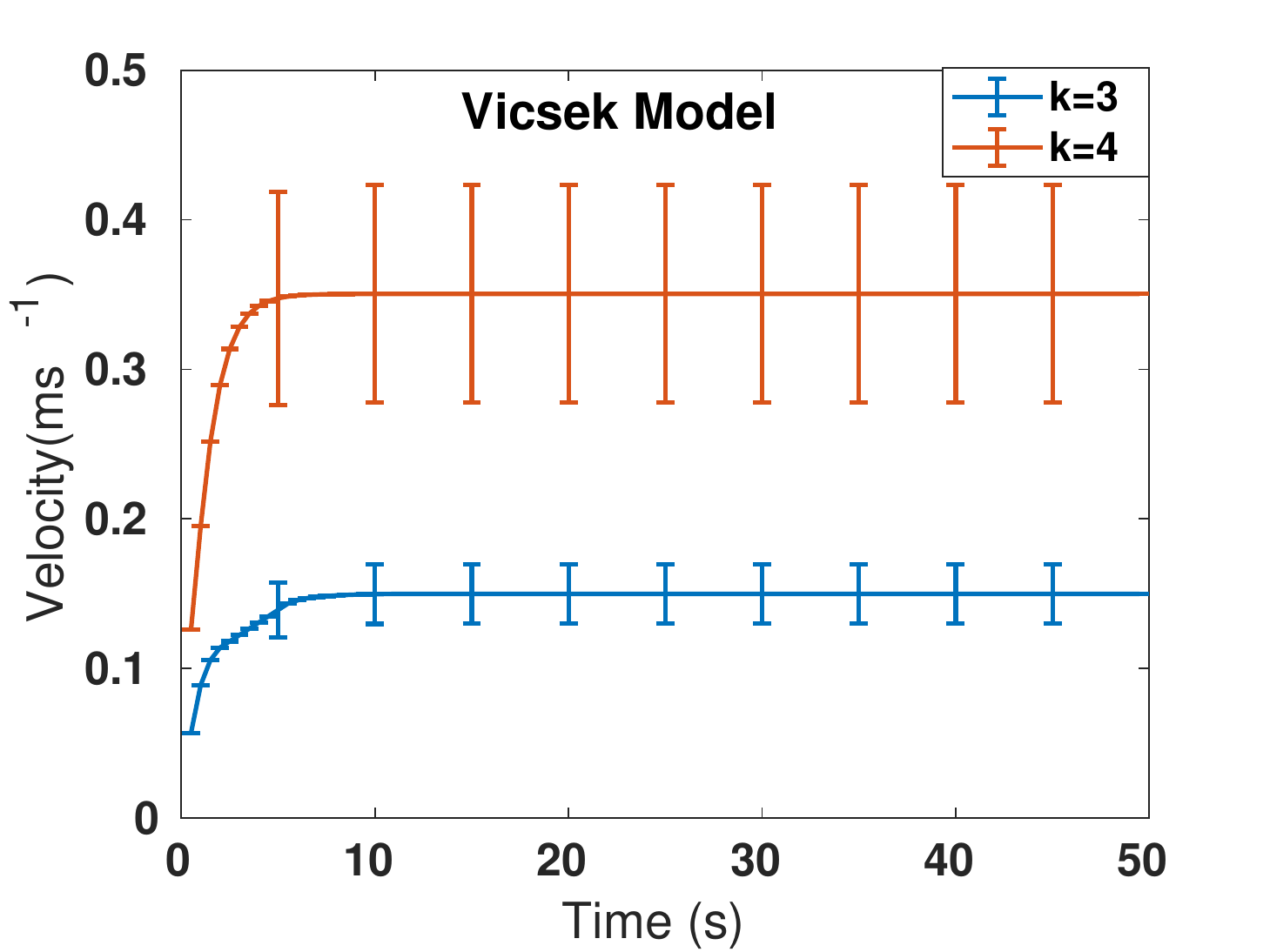}}
\caption{(a). Convergence rate of MFA. (b) Convergence time against the size of the search space and $k$. (c)(d) Average distance between the robots against the time. (e)(f) Average difference of velocities against time. }
\label{fig:convergence}
     \vspace{-0.5cm}
\end{figure}

We experimented using the proposed method with teams of 3-10 quadrotors in physical and simulation environments.
Each robot is armed with the tracking controller presented in \cite{lee2010geometric}, implemented on the Robot Operating System (ROS).
% For the localization of physical robots, we used an external motion capture system.
For the hardware, we opted for Crazyflie 2.0 nano-drones with an external motion capture system for localization.
For the experiments, we generated the trajectories with second-order ($n=2$) control inputs, according to Eq. \eqref{eq:l}.
% For such control inputs, typically, a robot's state space consists of the position, velocity, and acceleration values.
% We let each quadrotor drone communicate the state information with its nearest $k$ neighbors. 
Fig. \ref{fig:system} visualizes the main components of our system and their connections. 
% More specifically, each drone has its flocking controller to ensure the distributed nature of the algorithm.
The flocking algorithm and the inter-vehicle communication modules run at a fixed rate of $1/T_H$ and the latter depends on the ROS communication layer. 
For physical experiments, we used radio communication to transfer the resulting control commands to the root robots.
Typically, the internal controllers of the drones run at loop rates higher than $250Hz$ for attitude control; however, we executed the inter-vehicle communication and flocking algorithm modules at a much slower and more realistic $1/T_H = 5Hz$ rate.
To run the system with all the drone nodes and radio communications, we used a Linux system with an Intel $i7$ CPU and 16GB memory.
Further, all the components were implemented in the C/C++ programming language.  
In addition, we used $u_{max} = 1ms^{-2}$ and limited the maximum velocity of the robots to $1ms^{-1}$. 

% \subsection{Results}

We assessed the convergence of the flocking algorithm for different search spaces with varying discretization steps, $\textbf{d}_u$. 
Fig. \ref{snapshots}(a) shows the initial, stationary robot distribution used for the experiments.
We modeled the pairwise potential function in Eq. \eqref{eq:pairwise} with $a=5,b=15,k_a = 1.5, k_r=0.5$ and used the origin as the roosting center, $C_R$. 
We observed that changes in these parameters do not significantly affect the convergence rate of the flocking algorithm. 
However, changing the local minimum characteristics of the Morse potential reflects in the distances among the robots in the resulting cohesive aggregation. 

% Fig. \ref{snapshots} shows snapshots for 8 quadrotors reaching a cohesion in simulation for $k=3$. 
Though individual robot controllers utilize incomplete information about the global state of the swarm, our algorithm achieved a formation and velocity consensus as in Figs. \ref{snapshots}(d)-(e). 
Fig. \ref{fig:convergence}(a) shows the convergence rate of the flocking algorithm against the neighborhood size. We used KL divergence to measure the distance between two adjacent target distributions ($Q_{old}(X), Q(X)$) in the approximation step.
For a given planning horizon, the algorithm converged rapidly under 10 approximation iterations throughout the trials. 
% As each planning horizon results in new approximations, any disparities between the true and the approximated posteriors drop at the end of the horizon.
Since the approximation step computes the expected pairwise energy of each trajectory for that of the neighbors, the computational cost grows quadratically with the neighborhood size.
However, the proposed algorithm successfully converged within $200$ms for considerably large search spaces and $k$, which led the swarm to consensus \ref{fig:convergence}(b).
% Fig. \ref{fig:convergence}(c) and (e) show the consensus reaching process for our method. 
We experimented with a search space consisting of 25 trajectories and 10 robots and observed that the robots reached cohesion with minimal velocity differences while avoiding collisions with each other. 
Once converged, the shape of the formation remained roughly the same while maintaining the average distance among the robots as constant (Figs. \ref{fig:convergence}(c), \ref{fig:trajs}(a)). 

We compared our method's consensus reaching process to the algorithm proposed in \cite{vicsek1995novel}, with the same topological and initial conditions. 
Briefly, self-propelled particles in the Vicsek model adjust their velocities and the heading directions in the neighborhood to reach the velocity consensus. 
Even though the Vicsek model reached the velocity consensus faster, Figs. \ref{fig:convergence}(c)-(e) depict that the proposed method performed reasonably well to achieve even tighter alignments.
Further, our method outperformed the Vicsek model by achieving cohesive formations, such as pentagons, by minimizing the interaction energy.
In contrast, the particles in the Vicsek model did not show such behavior; refer to Fig. \ref{fig:convergence}(d), as the original work overlooks the pairwise interactions \cite{olfati2006flocking}.
% Figs. \ref{fig:trajs}(a) and \ref{fig:trajs}(b) show the trajectories of the robots after reaching the consensus in 2-and 3-dimensional simulations. 
% The initial formations caused immediate expansions or contractions in the swarm as the robots converged to the minimum pairwise energy, also leading to the merging of subgroups of the robots initially.
% Once the consensus is achieved, robots followed the unary energy gradient and flocked around the roosting center (Figs. \ref{fig:trajs}(a)).
% When the flocking is not limited to a plane, the formation consensus was difficult to observe due to fragmentation caused by the highly combinatorial outcome space. 
We observed fragmentation in 3-dimensional formations caused by the highly combinatorial outcome space however, the motions visually resembled the collective behavior in natural swarms (Fig. \ref{fig:trajs}(b)).
Figs. \ref{fig:trajs}(c)-(d) show the initial and final robots' distributions for physical UAV experiments.
We observed that collision avoidance to dependent on the parameters of the interaction energy, mainly when operating in tight indoor spaces.
Although the interaction energy function in Eq.\eqref{eq:pairwise} considers a homogeneity assumption, we denote that the proposed framework is also extensible toward heterogeneous robot swarms with different attraction and repulsion kernels.
Further, when coordinating swarms in outdoor environments, with large inter-robot distances, it is possible to increase the planning horizon to discount communication and perception delays while keeping the collective behavior intact.
A supplementary video demonstration for this work can be found at: \href{https://youtu.be/KVkNUKgViSg}{https://youtu.be/KVkNUKgViSg}.

% https://youtu.be/KVkNUKgViSg.
% In the former, the robots maintained a visible formation and velocity consensus however, in the latter the formations varied with the time. 

% We observed smooth trajectories in UAV flocking control, in contrast to the oscillatory behaviors reported in \cite{vasarhelyi2014outdoor}.

% From \ref{fig:trajs}(a) it is clear that our method is capable of simulating flocking behavior while maintaining cohesive formations of robots over a roosting area.
\begin{figure}
    \centering
      \subfigure[]
 	{\includegraphics[width = 0.24\textwidth, trim={0cm 0.2cm 0cm 0cm}]{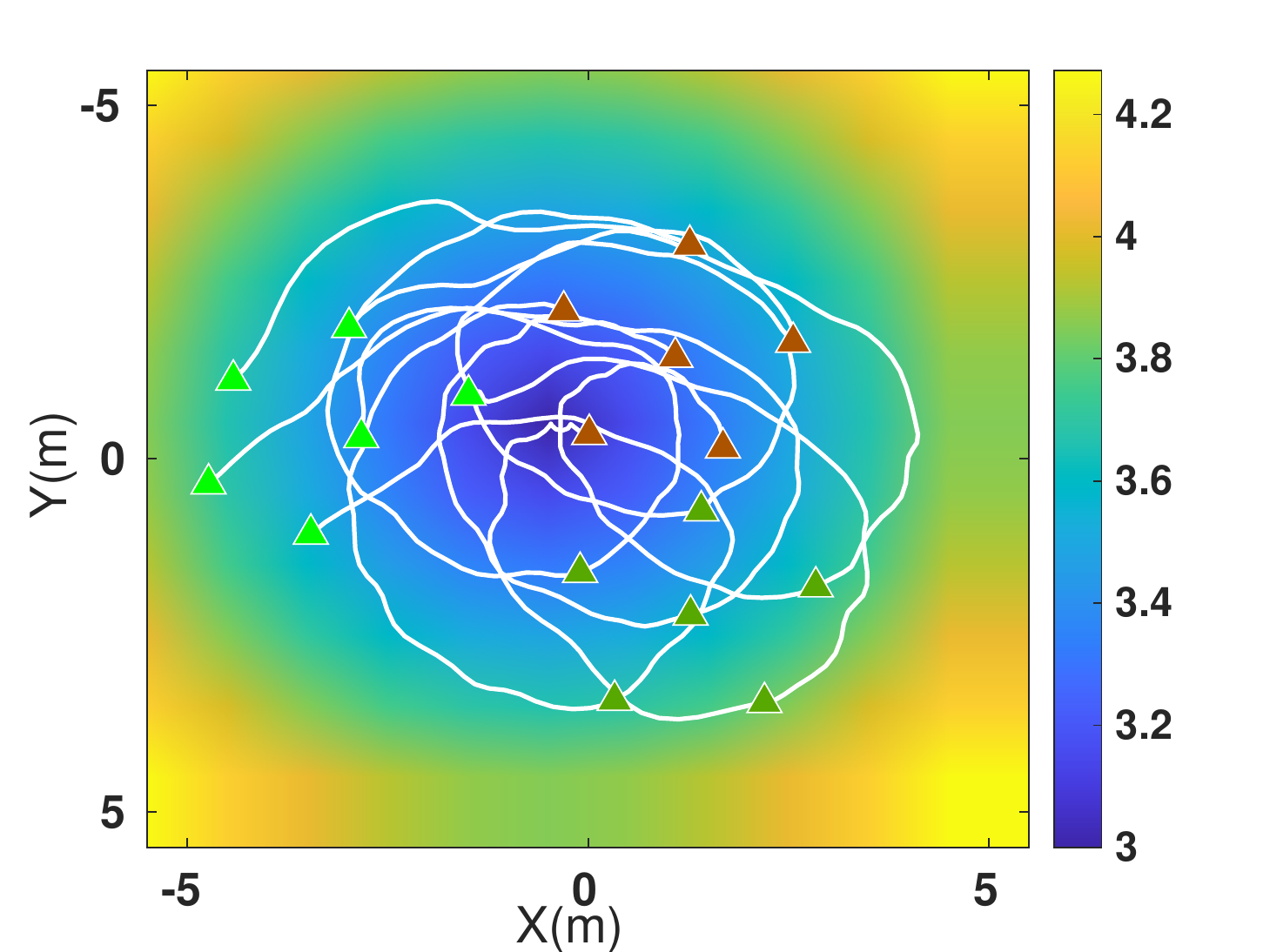}}
 \subfigure[]
 	{\includegraphics[width = 0.23\textwidth, trim={0cm 0.2cm 0cm 1cm}]{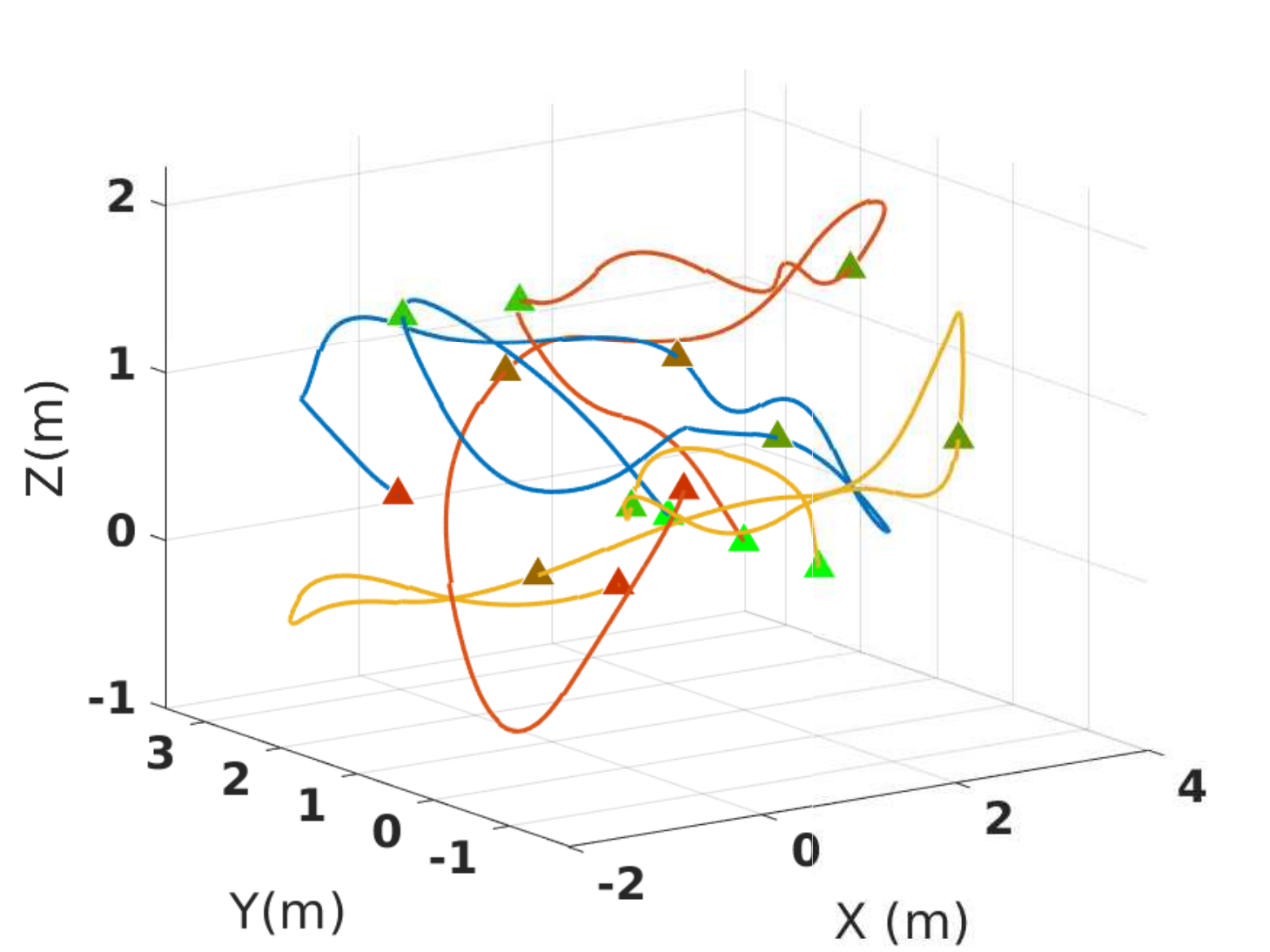}}
 	      \subfigure[]
 	{\includegraphics[width = 0.22\textwidth, trim={0cm 0.2cm 0cm 0.5cm}]{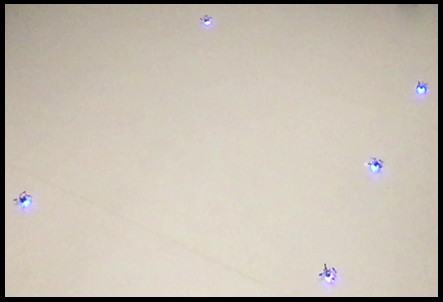}}
 \subfigure[]
 	{\includegraphics[width = 0.22\textwidth, trim={0cm 0.2cm 0cm 0.5cm}]{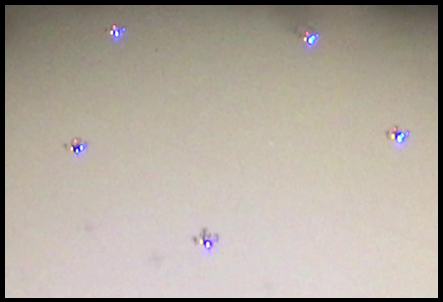}}
    \caption{(a) Trajectories of the robots after reaching consensus, plotted over the roosting energy function. (b). Trajectories of 3 robots in 3D space. Colors of the robots' markers change from green to red with time. (c) Initial and consensus reached formations (d) for five Crazyflie nano-drones.}
    \label{fig:trajs}
         \vspace{-0.5cm}
\end{figure}
\vspace{-0.1cm}
\section{Conclusions}
In this work, we have proposed a novel approach to simulating flocking behavior with UAVs by utilizing the local information. 
We have employed an interaction energy-based approach to represent the robots' relationships in the swarm.
Our method embeds the UAV dynamics into the outcome-space and results in feasible trajectories for the vehicles.
Consequently, our method eliminates the tedious parameter tuning phase of conventional flocking algorithms in practice.
The resulting algorithm builds on MFA and can control a swarm of UAVs to flocking consensus while avoiding collisions. 
Due to the fast converging nature, this approach suits online UAV flocking control usage.

\bibliographystyle{ieeetr}
\bibliography{root}

\begin{thebibliography}{10}

\bibitem{sharma2016uav}
V.~Sharma, M.~Bennis, and R.~Kumar, ``Uav-assisted heterogeneous networks for
  capacity enhancement,'' {\em IEEE Communications Letters}, vol.~20, no.~6,
  pp.~1207--1210, 2016.

\bibitem{tolstaya2020learning}
E.~Tolstaya, F.~Gama, J.~Paulos, G.~Pappas, V.~Kumar, and A.~Ribeiro,
  ``Learning decentralized controllers for robot swarms with graph neural
  networks,'' in {\em Conference on Robot Learning}, pp.~671--682, 2020.

\bibitem{honig2018trajectory}
W.~H{\"o}nig, J.~A. Preiss, T.~S. Kumar, G.~S. Sukhatme, and N.~Ayanian,
  ``Trajectory planning for quadrotor swarms,'' {\em IEEE Transactions on
  Robotics}, no.~99, pp.~1--14, 2018.

\bibitem{luis2020online}
C.~E. Luis, M.~Vukosavljev, and A.~P. Schoellig, ``Online trajectory generation
  with distributed model predictive control for multi-robot motion planning,''
  {\em IEEE Robotics and Automation Letters}, vol.~5, no.~2, pp.~604--611,
  2020.

\bibitem{zhu2019distributed}
H.~Zhu, J.~Juhl, L.~Ferranti, and J.~Alonso-Mora, ``Distributed multi-robot
  formation splitting and merging in dynamic environments,'' in {\em 2019
  International Conference on Robotics and Automation (ICRA)}, pp.~9080--9086,
  IEEE, 2019.

\bibitem{olfati2006flocking}
R.~Olfati-Saber, ``Flocking for multi-agent dynamic systems: Algorithms and
  theory,'' {\em IEEE Transactions on automatic control}, vol.~51, no.~3,
  pp.~401--420, 2006.

\bibitem{reynolds1987flocks}
C.~W. Reynolds, ``Flocks, herds and schools: A distributed behavioral model,''
  in {\em ACM SIGGRAPH computer graphics}, vol.~21, pp.~25--34, ACM, 1987.

\bibitem{okubo1986dynamical}
A.~Okubo, ``Dynamical aspects of animal grouping: swarms, schools, flocks, and
  herds,'' {\em Advances in biophysics}, vol.~22, pp.~1--94, 1986.

\bibitem{blume1971ising}
M.~Blume, V.~J. Emery, and R.~B. Griffiths, ``Ising model for the $\lambda$
  transition and phase separation in he 3-he 4 mixtures,'' {\em Physical review
  A}, vol.~4, no.~3, p.~1071, 1971.

\bibitem{krahenbuhl2012efficient}
P.~Kr{\"a}henb{\"u}hl and V.~Koltun, ``Efficient inference in fully connected
  crfs with gaussian edge potentials,'' {\em arXiv preprint arXiv:1210.5644},
  2012.

\bibitem{shankar2020mrfmap}
K.~S. Shankar and N.~Michael, ``Mrfmap: Online probabilistic 3d mapping using
  forward ray sensor models,'' {\em arXiv preprint arXiv:2006.03512}, 2020.

\bibitem{vicsek1995novel}
T.~Vicsek, A.~Czir{\'o}k, E.~Ben-Jacob, I.~Cohen, and O.~Shochet, ``Novel type
  of phase transition in a system of self-driven particles,'' {\em Physical
  review letters}, vol.~75, no.~6, p.~1226, 1995.

\bibitem{alonso2017multi}
J.~Alonso-Mora, S.~Baker, and D.~Rus, ``Multi-robot formation control and
  object transport in dynamic environments via constrained optimization,'' {\em
  The International Journal of Robotics Research}, vol.~36, no.~9,
  pp.~1000--1021, 2017.

\bibitem{vasarhelyi2018optimized}
G.~V{\'a}s{\'a}rhelyi, C.~Vir{\'a}gh, G.~Somorjai, T.~Nepusz, A.~E. Eiben, and
  T.~Vicsek, ``Optimized flocking of autonomous drones in confined
  environments,'' {\em Science Robotics}, vol.~3, no.~20, 2018.

\bibitem{fernando2019formation}
M.~Fernando and L.~Liu, ``Formation control and navigation of a quadrotor
  swarm,'' in {\em 2019 International Conference on Unmanned Aircraft Systems
  (ICUAS)}, pp.~284--291, IEEE, 2019.

\bibitem{turpin2013trajectory}
M.~Turpin, N.~Michael, and V.~Kumar, ``Trajectory planning and assignment in
  multirobot systems,'' in {\em Algorithmic foundations of robotics X},
  pp.~175--190, Springer, 2013.

\bibitem{8276634}
A.~Weinstein, A.~Cho, G.~Loianno, and V.~Kumar, ``Visual inertial odometry
  swarm: An autonomous swarm of vision-based quadrotors,'' {\em IEEE Robotics
  and Automation Letters}, vol.~3, no.~3, pp.~1801--1807, 2018.

\bibitem{shishika2017mosquito}
D.~Shishika and D.~A. Paley, ``Mosquito-inspired swarming for decentralized
  pursuit with autonomous vehicles,'' in {\em 2017 American Control Conference
  (ACC)}, pp.~923--929, IEEE, 2017.

\bibitem{vasarhelyi2014outdoor}
G.~V{\'a}s{\'a}rhelyi, C.~Vir{\'a}gh, G.~Somorjai, N.~Tarcai,
  T.~Sz{\"o}r{\'e}nyi, T.~Nepusz, and T.~Vicsek, ``Outdoor flocking and
  formation flight with autonomous aerial robots,'' in {\em 2014 IEEE/RSJ
  International Conference on Intelligent Robots and Systems}, pp.~3866--3873,
  IEEE, 2014.

\bibitem{hu2020vgai}
T.-K. Hu, F.~Gama, Z.~Wang, A.~Ribeiro, and B.~M. Sadler, ``Vgai: A
  vision-based decentralized controller learning framework for robot swarms,''
  {\em arXiv preprint arXiv:2002.02308}, 2020.

\bibitem{schilling2019learning}
F.~Schilling, J.~Lecoeur, F.~Schiano, and D.~Floreano, ``Learning vision-based
  flight in drone swarms by imitation,'' {\em IEEE Robotics and Automation
  Letters}, vol.~4, no.~4, pp.~4523--4530, 2019.

\bibitem{xi2006gibbs}
W.~Xi, X.~Tan, and J.~S. Baras, ``Gibbs sampler-based coordination of
  autonomous swarms,'' {\em Automatica}, vol.~42, no.~7, pp.~1107--1119, 2006.

\bibitem{fernando2020swarming}
M.~Fernando and L.~Liu, ``Swarming of aerial robots with markov random field
  optimization,'' {\em arXiv preprint arXiv:2010.06274}, 2020.

\bibitem{tanner2003stable}
H.~G. Tanner, A.~Jadbabaie, and G.~J. Pappas, ``Stable flocking of mobile
  agents part i: dynamic topology,'' in {\em 42nd IEEE International Conference
  on Decision and Control (IEEE Cat. No. 03CH37475)}, vol.~2, pp.~2016--2021,
  IEEE, 2003.

\bibitem{saber2003flocking}
R.~O. Saber and R.~M. Murray, ``Flocking with obstacle avoidance: Cooperation
  with limited communication in mobile networks,'' in {\em 42nd IEEE
  International Conference on Decision and Control (IEEE Cat. No. 03CH37475)},
  vol.~2, pp.~2022--2028, IEEE, 2003.

\bibitem{gazi2013lagrangian}
V.~Gazi, ``On lagrangian dynamics based modeling of swarm behavior,'' {\em
  Physica D: Nonlinear Phenomena}, vol.~260, pp.~159--175, 2013.

\bibitem{liu2017planning}
S.~Liu, M.~Watterson, K.~Mohta, K.~Sun, S.~Bhattacharya, C.~J. Taylor, and
  V.~Kumar, ``Planning dynamically feasible trajectories for quadrotors using
  safe flight corridors in 3-d complex environments,'' {\em IEEE Robotics and
  Automation Letters}, vol.~2, no.~3, pp.~1688--1695, 2017.

\bibitem{ballerini2008interaction}
M.~Ballerini, N.~Cabibbo, R.~Candelier, A.~Cavagna, E.~Cisbani, I.~Giardina,
  V.~Lecomte, A.~Orlandi, G.~Parisi, A.~Procaccini, {\em et~al.}, ``Interaction
  ruling animal collective behavior depends on topological rather than metric
  distance: Evidence from a field study,'' {\em Proceedings of the national
  academy of sciences}, vol.~105, no.~4, pp.~1232--1237, 2008.

\bibitem{shang2014influence}
Y.~Shang and R.~Bouffanais, ``Influence of the number of topologically
  interacting neighbors on swarm dynamics,'' {\em Scientific reports}, vol.~4,
  p.~4184, 2014.

\bibitem{carrillo2013new}
J.~A. Carrillo, S.~Martin, and V.~Panferov, ``A new interaction potential for
  swarming models,'' {\em Physica D: Nonlinear Phenomena}, vol.~260,
  pp.~112--126, 2013.

\bibitem{pivtoraiko2009differentially}
M.~Pivtoraiko, R.~A. Knepper, and A.~Kelly, ``Differentially constrained mobile
  robot motion planning in state lattices,'' {\em Journal of Field Robotics},
  vol.~26, no.~3, pp.~308--333, 2009.

\bibitem{koller2009probabilistic}
D.~Koller and N.~Friedman, {\em Probabilistic graphical models: principles and
  techniques}.
\newblock MIT press, 2009.

\bibitem{lee2010geometric}
T.~Lee, M.~Leoky, and N.~H. McClamroch, ``Geometric tracking control of a
  quadrotor uav on se (3),'' in {\em Decision and Control (CDC), 2010 49th IEEE
  Conference on}, pp.~5420--5425, IEEE, 2010.

\end{thebibliography}

\end{document}